\documentclass{article}
\usepackage{iclr2027_conference,times}
\usepackage{microtype}

\usepackage{amsmath,amsfonts,bm}

\def\eqref#1{equation~(\ref{#1})}

\def\1{\bm{1}}

\DeclareMathAlphabet{\mathsfit}{\encodingdefault}{\sfdefault}{m}{sl}
\SetMathAlphabet{\mathsfit}{bold}{\encodingdefault}{\sfdefault}{bx}{n}

\newcommand{\E}{\mathbb{E}}

\newcommand{\numpaths}{K}                 %
\newcommand{\vocab}{\mathcal{V}}          %
\newcommand{\logit}{\ell}                 %
\newcommand{\logitvec}{\boldsymbol{\ell}} %
\newcommand{\divfn}{\Delta}              %
\newcommand{\cnt}{c}                      %
\newcommand{\cntvec}{\mathbf{c}}          %
\newcommand{\pbar}{\bar{p}}               %
\newcommand{\dose}{\lambda}               %
\newcommand{\efftilt}{\tilde{\lambda}}    %
\newcommand{\temp}{\tau}                  %

\newcommand{\plur}{\mathrm{plur}}         %
\newcommand{\passat}{\mathrm{pass}}       %

\newcommand{\kl}[2]{\mathrm{KL}\!\left(#1 \,\middle\|\, #2\right)}
\newcommand{\partf}{Z}                    %
\newcommand{\gate}{g}                      %

\usepackage{hyperref}
\usepackage{url}
\usepackage{amsmath}
\usepackage{amssymb}
\usepackage{amsthm}
\usepackage{booktabs}
\usepackage{float}
\floatstyle{ruled}
\newfloat{algorithm}{t}{loa}
\floatname{algorithm}{Algorithm}
\usepackage{algpseudocode}
\usepackage{graphicx}
\usepackage{xcolor}
\usepackage{array}

\usepackage{pifont}
\usepackage{xspace}
\usepackage{listings}

\lstdefinestyle{trace}{basicstyle=\ttfamily\scriptsize,backgroundcolor=\color{black!3},frame=single,framerule=0.4pt,rulecolor=\color{black!12},framesep=5pt,columns=fullflexible,keepspaces=true,breaklines=true,breakindent=0pt,xleftmargin=0pt,aboveskip=6pt,belowskip=10pt,upquote=true}

\newcommand{\circled}[1]{\ding{\the\numexpr171+#1\relax}}

\newcommand{\methodnamelong}{Self-Repulsion\xspace}
\newcommand{\methodname}{SR\xspace}

\theoremstyle{plain}
\newtheorem{proposition}{Proposition}

\newtheorem{lemma}[proposition]{Lemma}
\newtheorem{corollary}[proposition]{Corollary}
\theoremstyle{definition}

\newtheorem{remark}[proposition]{Remark}

\title{Self-Repulsive Sampling for \\ Diffusion Language Models}

\author{Michael Helcig \\
ETH Zürich \\
\texttt{mhelcig@ethz.ch} \\
\And
Martin Jaggi \\
EPFL \\
\texttt{martin.jaggi@epfl.ch}
}

\iclrfinalcopy
\begin{document}

\maketitle
\lhead{Preprint}

\begin{abstract}
Sampling several responses and voting over their answers can improve a language model's accuracy, but repeated answers limit the benefit of additional samples. Raising temperature increases diversity at a potential cost to per-sample accuracy. We introduce \methodnamelong (\methodname), a sampler for masked diffusion language models that uses peer commitments to diversify the pool. At each penalized denoising step, each path lowers a token's logit according to how many peers have committed that token at the same position. Paths share a batched forward pass and then commit in sequence, so later paths observe choices made earlier in the same step. This coupling requires no training or additional forward or backward pass and can produce distinct paths even at temperature zero. When all paths commit a position together from identical logits, the update exactly maximizes total logit minus a convex duplication cost. On LLaDA-8B-Instruct with ten paths and 128 denoising steps, deterministic \methodname reaches $80.38\%$ plurality accuracy on GSM8K, compared with $70.17\%$ for the unpenalized greedy decoder. At temperature $0.6$ and matched model-evaluation budgets, the count penalty improves over self-consistency by $2.06$ percentage points in blocks of 32 and $14.50$ under pure diffusion. Experiments on GSM8K, MATH and TruthfulQA show that voting gains arise mainly from higher coverage of correct answers, with gains that vary by benchmark and decoding regime.

\end{abstract}

\section{Introduction}
\label{sec:intro}

Sampling several reasoning paths and voting over their answers is a standard way to improve a language model's accuracy with additional inference computation. Self-consistency \citep{wang2023selfconsistency} samples chain-of-thought paths independently \citep{wei2022chain} and returns the most frequent final answer, without training a verifier. Its success depends on the samples available to the vote. In particular, \emph{coverage}, the probability that the pool contains a correct answer, upper-bounds the accuracy of any selector over that pool. Repeated answers can limit the benefit of additional samples, but disagreement alone is insufficient: a more diverse pool need not yield a more accurate vote \citep{bay2026sampling,olausson2026twotemps}. The sampling problem is therefore to find more correct answers while retaining enough evidence for the vote to select them.

Temperature offers a simple way to change the pool. At low temperature, samples concentrate on a few answers; at higher temperature, they explore more alternatives but can lose per-sample accuracy. Temperature and truncation act on each sample's own distribution, regardless of what other samples have already produced. Couplings based on shared randomness can spread the draws while preserving their marginal distributions \citep{vilnis2023arithmetic,parashar2024quasi}. We investigate a different signal: the tokens that other paths have already committed while the ensemble is being generated. Can this information improve coverage without relying on additional sampling noise?

Masked diffusion language models (MDLMs), a prominent form of discrete diffusion, are an emerging alternative to autoregressive generation \citep{li2025dlmsurvey}. A response begins fully masked and the decoder runs $T$ denoising steps, each a single forward pass, or \emph{network function evaluation} (NFE). Each step predicts every still-masked position in parallel \citep{ye2025dream} and commits the most confident \citep{nie2025llada}. Generation costs $T$ NFEs rather than one forward pass per token, which industrial systems increasingly build on \citep{inception2025mercury,song2025seeddiffusion,diffusiongemma2026}. This structure makes an alternative source of ensemble diversity available at almost no cost. When $\numpaths$ samples decode the same positions on a shared schedule, each reads its peers' current commitments from the batch the decoder already holds, while the pool is still forming; an any-order decoder can then defer a contested position and repair around a substitution, which a fixed-order decoder cannot. This signal has a timing. Between steps the pool is fixed, so a penalty read at the \emph{start} of a step sees only the commitments of earlier steps. Read \emph{within} a step, as the paths commit in turn, it also sees the tokens peers commit earlier in that same step. \textbf{Released decoders leave this signal unused;} in the case of LLaDA they even use greedy (argmax) decoding by default \citep{nie2025llada}.

We introduce \methodnamelong (\methodname), a sampler that uses this within-step information. At each penalized step, a path lowers a token's logit in proportion to the number of peers that have committed that token at the same position. All paths share one batched forward pass, then commit in a fixed order; each path reads the updated peer counts before making its choices. We call this ordered update the \emph{cascade}. The confidence-based commit rule and final plurality vote remain unchanged (Figure~\ref{fig:mechanism}, Section~\ref{sec:method}). The penalty requires no training, learned potential, or additional forward or backward pass.

Temperature zero isolates the role of the cascade. With identical initial canvases, a shared deterministic update and path-independent tie-breaking, paths that read the pool only at the start of a step remain identical (Proposition~\ref{prop:degeneracy}). The cascade can break this symmetry because later paths observe earlier commitments. Deterministic ensembles and diversity penalties have precedents in CoT-decoding, Diverse Beam Search and Orthogonal Diverse Diffusion (ODD) \citep{wang2024cotdecoding,vijayakumar2018diverse,lamont2026odd}; our contribution is to couple masked-diffusion paths through their current commitments at each position and evaluate the resulting pools under a fixed voting rule.

\paragraph{Contributions:}
\begin{itemize}
  \item \textbf{\methodnamelong (Section~\ref{sec:method}).} A drop-in sampler that draws ensemble diversity from the pool, not from noise: at each still-masked position it lowers each token's logit in proportion to how strongly the peers have committed to that token, adding no extra forward or backward pass on a shared schedule.
  \item \textbf{Ensembles without randomness.} At temperature zero a symmetric start-of-step sampler is degenerate: its $\numpaths$ samples coincide, so the vote cannot beat a single decode. Reading the peers within the step can make the paths differ. When all paths commit a position together from identical logits, the cascade maximizes total logit minus a convex cost of duplication. Measured, they vote above self-consistency with no randomness, on GSM8K reaching $80.38$ against self-consistency's $77.05$ and the greedy decode's $70.17$ (128 steps). %
  \item \textbf{Experiments on masked diffusion (Section~\ref{sec:experiments}).} On LLaDA-8B-Instruct over GSM8K, MATH and TruthfulQA, \methodname matches or exceeds the diversity methods we reproduce at matched computation, the gain in coverage, not the vote, and in blocks 2.46 points above a peer-blind control at matched disagreement (Appendix~\ref{sec:results-argmax}). Its gain over self-consistency is larger where duplication is worse: on GSM8K, $+2.06$ points in blocks of 32 and $+14.50$ under pure diffusion, at the same 1280 evaluations.
\end{itemize}

\section{\methodnamelong}
\label{sec:method}

In the following, we characterize \methodnamelong from three angles. As a \emph{sampler}, it is a single additive term on the logits, by which the $\numpaths$ paths of one parallel decode repel one another through their commitments, with no added randomness and no extra forward or backward pass (Section~\ref{sec:method-sampler}). As a \emph{penalty}, at temperature zero it \emph{allocates} the paths at a position they commit together from identical logits, the cascade maximizing an explicit quality-minus-duplication objective (Section~\ref{sec:method-positive}). As one \emph{sampler among others}, its asymmetry is the information order alone, the samples reading one another as the pool forms; at temperature zero a peer-blind ensemble collapses, which reading the peers within the step escapes (Section~\ref{sec:method-boundary}).

\subsection{Sampler}
\label{sec:method-sampler}

\methodnamelong runs $\numpaths$ masked-diffusion \emph{paths} on one prompt in parallel. Each path fills its own \emph{canvas} and the $\numpaths$ canvases together form the \emph{pool}. The paths are coupled through their logits: at every denoising step, at each still-masked position $j$, the law of each path is tilted away from the tokens \emph{peers}, the other paths, have committed there. Within a step the paths commit in turn, each reading the tokens the paths ahead of it committed; we call this ordered commit the \emph{cascade}. Writing~$p$ for the law the path would otherwise sample from and $\divfn$ for a non-negative statistic of the peers' commitments at $j$, for $\temp > 0$ the path samples from
\begin{equation}
  \label{eq:tilt}
  q(v) \;\propto\; p(v)\,\exp\!\left(-\efftilt\,\divfn_v\right) .
\end{equation}
The $\numpaths$ answers are then aggregated by a plurality vote. Self-consistency is the $\dose = 0$ member of this family, with the penalty disabled; this section asks what a positive strength provides.

\paragraph{Cascade.} When each path reads its peers is the one choice in the sampler that is not forced, and it is what makes the method work. The natural implementation reads $\divfn$ once after the batched forward pass and commits all $\numpaths$ paths in parallel; then $\divfn$ is fixed at the start of the step, identical across the paths, so at $\temp = 0$ they commit the same tokens and never diverge (Proposition~\ref{prop:degeneracy}). The cascade instead reads $\divfn$ within the per-path loop, so a path sees what the paths ahead of it committed at this same step. Formally, the $\numpaths$ paths share a common initial canvas, each is scored on its own, and each commits by argmax under an index-free tie-break after one shared update from a statistic $\divfn$ of the pool.

\begin{proposition}[Degeneracy at zero temperature]
  \label{prop:degeneracy}
  At $\temp = 0$, if each path's $\divfn$ is one function, shared across paths, of its own canvas and its peers' canvases, read at the start of the step, the canvases coincide at every step and the pool holds one distinct sample; self-consistency at $\temp = 0$ is the case $\divfn \equiv 0$.
\end{proposition}

\paragraph{Setup.} The sampler adds a peer-count penalty to the logits and reads these counts within the ordered commitment loop, while retaining the original decoder's confidence-based commit rule and denoising schedule. With raw logits $\logitvec$ and strength $\dose \ge 0$, every path applies $\logit_k[j,v] \leftarrow \logit_k[j,v] - \tfrac{\dose}{\numpaths - 1}\, \divfn_v$. This splits the strength evenly across the $\numpaths - 1$ peers, so a unanimous pool costs $\dose$ whatever $\numpaths$ is. For $\temp > 0$, dividing by the temperature $\temp$ carries the penalty into the exponent and recovers \eqref{eq:tilt}. We distinguish three related quantities:
\begin{equation}
  \label{eq:strengths}
  \text{strength } \dose, \qquad
  \text{gate } \gate = \dose/(\numpaths-1), \qquad
  \text{tilt } \efftilt = \gate/\temp \quad (\temp > 0).
\end{equation}
At the deployed $\numpaths = 10$ and $\temp = 0.6$, the gate is $\gate = 7.11$ (strength $\dose = 64$, tilt $\efftilt = 11.85$). Across $\numpaths$ we hold the strength $\dose$, not the gate: $\gate\numpaths = \dose\numpaths/(\numpaths-1)$ tends to $\dose$, so the fractional allocation proportions approach a $\numpaths$-independent limit at fixed $\dose$ and flatten toward uniform at fixed $\gate$ (Appendix~\ref{app:abl-k}); the ablation tables report $\dose$ and $\gate$ together.

The penalty is subtracted from the raw logits, before dividing by $\temp$, so for $\temp > 0$, $q$ is an exponential tilt of $p$ that strengthens as $\temp$ falls. The logit penalty remains defined at $\temp = 0$. At $\temp = 0$ the decoder takes $\arg\max_v (\logit_v - \gate\,\cnt_v)$, with $\cnt_v$ the number of peers committed to token $v$ at $j$, so the gate is a margin in raw logits: the top token loses once $\gate$ times its excess peer count over a competitor exceeds its raw-logit lead over that competitor. At the deployed $\temp = 0.6$ the tilt is $\efftilt = 11.85$, so a token even one peer holds is down-weighted by $e^{-\efftilt} \approx 7 \times 10^{-6}$, close to hard exclusion unless its raw-logit lead over the best untaken token exceeds the gate (Appendix~\ref{app:example}).

\begin{figure}[!ht]
  \centering
  \includegraphics[width=0.95\linewidth]{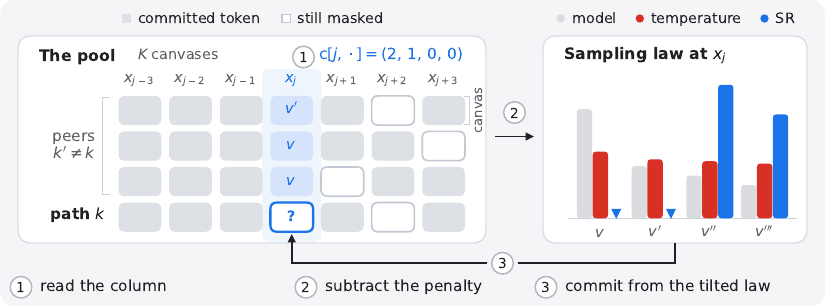}
  \caption{One denoising step of \methodnamelong, four paths, illustrative values. \textbf{Left:} the $\numpaths$ paths decode in parallel; the column at a still-masked position $j$ is the count vector $\cntvec_k[j,\cdot]$, the peers committed to each token (here $(2,1,0,0)$). The step runs in three parts: (1)~read the column, (2)~subtract the penalty $\gate\,\cntvec_k[j,\cdot]$ from the logits, (3)~commit from the tilted law. \textbf{Right:} the sampling law at $x_j$ under the model, under temperature and under \methodname. The tilt down-weights tokens peers hold and preserves probability ratios between tokens with equal counts; temperature reshapes every token, the taken ones included. Values are schematic except the effective tilt $e^{-11.85}$ per peer.}
  \label{fig:mechanism}
\end{figure}

\paragraph{Penalty statistic.} We deploy the \emph{count} statistic, $\divfn_v = \cnt_k[j,v]$, the number of peers committed to token $v$ at position $j$ when path $k$ reads the pool. We suppress $k$ and $j$ below. Appendices~\ref{app:abl-sweeps} and~\ref{app:abl-crossover} give two variants that read the peers' predicted distributions instead, which can down-weight a token before any peer commits it but degenerate at $\temp = 0$, where those distributions are start-of-step state (Proposition~\ref{prop:degeneracy}); each induces \eqref{eq:tilt}. Nothing here is specific to masked diffusion beyond what an any-order decoder does with a pushed-off path, which a fixed-order decoder cannot. It can defer a contested position: a path pushed off a peer-held token scores that position by its replacement's lower untilted confidence and so defers it, and at $\temp = 0$ and gate 8 only 3.00\% of the commits made under the penalty in blocks of 32 differ from the path's own argmax, each where a peer already holds that argmax. And it can repair around a substitution, the neighbors still masked and attention bidirectional, so a perturbation stays local (Appendix~\ref{app:example}); the confidence order helps both, penalizing early structural commits and releasing before the late uncertain ones. Masked diffusion is the home of the method because it can defer and repair, not because it keeps the logits shared. In a fixed-order decoder this within-step count read is Diverse Beam Search's Hamming penalty at one path per group \citep{vijayakumar2018diverse}; the two properties above are what any-order decoding adds.

\paragraph{Cost.} \methodname adds no extra forward pass, no backward pass, no learned potential and no training. The paths share a canvas layout and block schedule, so peer counts can be computed directly from the commitments already stored in the batch. The added computation consists of reading these counts and subtracting the corresponding penalty from each path's logits. The original decoder's commit and remasking rules are unchanged: positions are ranked by the untilted probability of the selected token. Changing that token can therefore change when a position is committed. The penalty is active during the first three-quarters of each block's denoising steps and disabled thereafter. Algorithm~\ref{alg:sr} gives the full procedure.

\begin{algorithm}[t]
\caption{\methodnamelong. Shading marks the peer-count read and logit penalty, being the only modification to self-consistency. Counts are read \emph{inside} the cascade, so later paths see earlier commitments from the same step. $s$ is the prompt, $j$ indexes positions, and $v$ indexes tokens. Token updates apply only at still-masked positions.}
\label{alg:sr}
\small
\begin{algorithmic}[1]
\Require model $f_\theta$, prompt $s$, paths $\numpaths > 1$,
gate $\gate = \dose/(\numpaths - 1)$, temperature $\temp$, steps $T$,
penalty scope $\rho$ (active during the first fraction $\rho$ of each block)

\State $\mathbf{x}_k \gets (\texttt{[MASK]}, \dots, \texttt{[MASK]})$
for every path $k \le \numpaths$

\For{$t = 1, \dots, T$}
  \State $\logitvec_k \gets f_\theta(s, \mathbf{x}_k)$
    for all $k$
    \Comment{one batched forward pass}

  \For{$k = 1, \dots, \numpaths$}
    \Comment{cascade, fixed path order}

    \State \colorbox{blue!6}{\strut
      $\cnt_k[j,v] \gets
      \#\{k' \ne k : x_{k'}[j] = v \ne \texttt{[MASK]}\}$}
      at masked $j$ \Comment{sees this step's earlier commits}

    \State \colorbox{blue!6}{\strut
      $\logitvec'_k \gets
      \logitvec_k - \gate\,\cntvec_k$
      if the penalty is active, else $\logitvec_k$}

    \State $\tilde{x}_k[j] \gets \arg\max_v \logit'_k[j,v]$
      if $\temp = 0$, else
      $\tilde{x}_k[j] \sim
      \mathrm{softmax}(\logitvec'_k[j,\cdot]/\temp)$

    \State commit $\tilde{x}_k[j]$ where the original confidence rule
      selects (scored by untilted probability); other positions stay masked
  \EndFor
\EndFor

\State \Return
  $\arg\max_a \sum_k
  \mathbf{1}\{\mathrm{answer}(\mathbf{x}_k) = a\}$,
  ties uniform
  \Comment{plurality vote}
\end{algorithmic}
\end{algorithm}

The closest analogue, ODD, spends a gradient step on the logits per step, orthogonalizing each sample in a pooled feature space \citep{lamont2026odd}; particle guidance needs a fixed or learned potential \citep{corso2024particle}; determinantal beam search needs a similarity kernel and greedy log-determinant maximization \citep{meister2021determinantal}. \methodname needs only the resident batch column its paths already share, without additional model evaluations.

The processing order inside the cascade, held fixed across steps, is a free choice that does not change the pool: at temperature zero every such order yields the same multiset (Remark~\ref{prop:canonicity}, Appendix~\ref{app:proofs}).

\subsection{Penalty}
\label{sec:method-positive}

At temperature zero the penalty \emph{allocates}: where the paths commit a position together from identical logits, the cascade maximizes total logit minus a duplication penalty. Proofs are in Appendices~\ref{app:allocation} and~\ref{app:sparsemax}.

\paragraph{Allocation at a shared position.} Suppose all $\numpaths$ paths commit the same position in one step and face identical logits. Here $\cnt_v$ counts the paths allocated to token $v$. Its successive choices have scores $\logit_v,\ \logit_v-\gate,\ \logit_v-2\gate,\ldots$: each commitment lowers that token's next score by $\gate$. The cascade takes the $\numpaths$ largest scores across these descending sequences. The \emph{water level} is a common cutoff separating selected from unselected scores; it is determined directly by the resulting allocation (Figure~\ref{fig:waterlevel} in Appendix).

\begin{proposition}[Allocation at a shared position]
  \label{prop:allocation}
  At $\temp = 0$, at a step where all $\numpaths$ paths commit the same position, none having committed it before, and face the same logits $\logitvec$, the commitment counts $\cntvec$ exactly maximize $\sum_v \cnt_v \logit_v - \gate \sum_v \binom{\cnt_v}{2}$ over $\{\cntvec \in \mathbb{Z}_{\ge 0}^{\vocab} : \sum_v \cnt_v = \numpaths\}$, by water-filling against a common level: the penalty allocates the paths by trading total logit against the pool's duplication, one gate $\gate$ per agreeing pair. Away from shared logits each path still commits $\arg\max_v (\logit_k[j,v] - \gate\, \cnt_v)$, but the global optimality guarantee no longer applies.
\end{proposition}
\emph{Reading.} A path switches tokens when another token has a higher penalized score. The allocation need not make every path distinct. The result holds for the count statistic, which charges only committed tokens; the distribution-reading variants are not covered. Its hypothesis is narrow: in the decodes above, all $\numpaths$ paths commit a position at one step at only 4.9\% of positions, fewer still under full canvas, the regime with the largest gain.

\begin{corollary}[Sparsemax form of the fractional allocation]
  \label{cor:sparsemax}
  For $\gate > 0$, allow the counts in
  Proposition~\ref{prop:allocation} to take nonnegative real
  values while retaining $\sum_v \cnt_v = \numpaths$.
  The unique optimum is
  \[
    \cntvec^\star
    = \numpaths\,\mathrm{sparsemax}
      \!\Big(\frac{\logitvec}{\gate\numpaths}\Big),
  \]
  where $\mathrm{sparsemax}$ denotes Euclidean projection onto
  the probability simplex
  $\{\mathbf{p}\ge 0:\sum_v p_v=1\}$
  \citep{martins2016sparsemax}.
  The cascade's integer counts satisfy
  $|\cnt_v-\cnt_v^\star|<1$ for every token $v$
  (Appendix~\ref{app:sparsemax}).
\end{corollary}
\emph{Reading.} The fractional solution is $\cnt_v^\star=\max\{0,(\logit_v-\mu)/\gate\}$, where the threshold $\mu$ is chosen so that the counts sum to $\numpaths$. Tokens at or below $\mu$ receive zero allocation;
above it, allocation is proportional to the excess logit.
The cascade assigns whole paths, within one path per token of this fractional solution. The continuous threshold $\mu$ need not equal the integer cutoff described above.

\paragraph{Tilt.} With randomness, unlike temperature, the tilt raises untaken tokens and lowers those of the largest count, preserving equal-penalty ratios (Figure~\ref{fig:mechanism}, right; Appendix~\ref{app:tilt}). Like any exponential tilt, it is the I-projection of $p$ onto laws with at most its expected peer agreement (Lemma~\ref{lem:iprojection}; \citealp{csiszar1975}).

\subsection{Distinction to other samplers}
\label{sec:method-boundary}

\paragraph{Sources of asymmetry.} For the $\numpaths$ paths to decode differently, the decode must depend on which path is running. At temperature zero, with no randomness, only three parts of the decode can carry that dependence:
\begin{itemize}
  \item the \emph{rule} each path applies; different rules give an ensemble of configurations \citep{lee2025hex}, or an index-ordered penalty breaks the symmetry \citep{lamont2026odd};
  \item the \emph{assignment} of tokens to paths; handing out shared candidates by rank is a search \citep{wang2024cotdecoding};
  \item the \emph{information order}: what each path reads of its peers and when.
\end{itemize}
A \emph{drop-in} sampler, one shared configuration and one additive penalty, forgoes the first two; the cascade realizes the information order by committing the paths in sequence, each reading its predecessors, though not uniquely (Appendix~\ref{sec:results-symmetry}).

\paragraph{Peer-blind boundary.} A \emph{peer-blind} sampler transforms each sample without reading its peers, drawing its randomness independently across samples. A pool can differ from an independent self-consistency pool in only two ways: a different per-sample law, or correlation. Writing $P$ for the joint law, $P_k$ for its marginals and $p$ for one untilted sample, the divergence splits,
\[
  \kl{P}{p^{\otimes \numpaths}} \;=\; \sum_{k} \kl{P_k}{p} \;+\; \kl{P}{\textstyle\bigotimes_k P_k} ,
\]
into a \emph{marginal} term (the samples' own laws) and a \emph{dependence} term (their correlation, a total correlation). A peer-blind sampler leaves the dependence term at zero: a product channel carries a product law to a product law, so its $\numpaths$ draws are i.i.d.\ from one law, and their mean pairwise answer agreement within a pool equals the mean across independent pools, $\E[M_{\mathrm{within}}] = \E[M_{\mathrm{cross}}]$. The split is a statement about the sampled configurations; against a deterministic pool its marginal terms are infinite, and $\temp = 0$ is covered by Proposition~\ref{prop:degeneracy} instead. What a peer-blind sampler can still move is the marginal term, and coupling the draws at fixed marginals raises coverage by at most a bounded amount.

\begin{proposition}[Coupling ceiling on coverage]
  \label{prop:frechet}
  At fixed per-sample marginals, re-coupling raises coverage above its independent value by at most $(1 - 1/\numpaths)^{\numpaths} < e^{-1}$.
\end{proposition}
\emph{Reading.} Correlating the draws without changing their laws adds at most $0.3487$ coverage per problem at $\numpaths = 10$ (Appendix~\ref{app:frechet}), so any gain above that ceiling must change the laws. Temperature, the truncation family and any self-computed penalty ($\divfn = p$) are peer-blind, their pools i.i.d.; Section~\ref{app:abl-conditionality} tests whether a self-computed penalty tracks temperature.

\section{Experiments}
\label{sec:experiments}

We test whether peer commitments can diversify deterministic decoding, how repulsion compares with increasing temperature, and whether the resulting pools improve plurality accuracy. We compare with self-consistency and reproduced diversity methods at matched numbers of model evaluations, and use controls to examine the role of peer information. Additional results are in Appendix~\ref{app:further}.

\paragraph{Metrics.} \emph{Plurality accuracy} scores the most frequent parsed answer, averaging uniformly over tied answers. \emph{Pass@$k$} measures whether a subset of $k$ paths contains a correct answer; pass@1 is per-sample accuracy and pass@10 is coverage. \emph{Disagreement} is the rate at which two paths return different parsed answers. Metrics are computed per problem and then averaged (Appendix~\ref{app:metrics}); scores and differences are reported as percentages and percentage points, respectively.

\paragraph{Setup.}
We evaluate on GSM8K \citep{cobbe2021training}, MATH \citep{hendrycks2021math} and TruthfulQA \citep{lin2022truthfulqa}. All ensemble methods use LLaDA-8B-Instruct in BF16 with $\numpaths = 10$ paths and generation length 256. Appendix~\ref{app:llada15} adds a second checkpoint, LLaDA-1.5 \citep{zhu2025llada15}, with similar GSM8K gains over self-consistency (coverage $+6.75$, plurality $+2.46$). Unless specified otherwise, \methodname uses the count statistic with strength $\dose = 64$ (gate $\gate = 7.11$; recommended settings in Table~\ref{tab:abl-defaults}). We compare pure diffusion, which treats the response as one block, with semi-autoregressive blocks of 32, reported to improve LLaDA-Instruct on GSM8K and MATH \citep{nie2025llada}. Strength was selected on 200 problems with one seed; the penalty scope was fixed beforehand to the first three-quarters of each block's steps ($\rho = 0.75$; Appendix~\ref{app:ablations}). Sampled configurations use $\temp = 0.6$ unless a temperature sweep is stated. The released decoder instead uses greedy decoding with length and step count 512 \citep{nie2025llada}. Table~\ref{tab:nfe} uses 128 denoising steps per path (1280 evaluations in total); the temperature sweeps, Table~\ref{tab:argmax} and Table~\ref{tab:main} use 256 steps (2560 evaluations). These budgets count model evaluations and exclude sampler-specific overhead.

\paragraph{Baselines.}
TAPS perturbs prompt embeddings with annealed noise \citep{wu2026taps}; ODD diversifies samples through a penalty on their pooled features \citep{lamont2026odd}. We hold the model, prompt, length, schedule, path count, step count and aggregation fixed, and vary the sampler. This protocol excludes the schedule ensemble of \citet{lee2025hex}, which varies decoding configurations. All baseline numbers are our reproductions. We also vary each family's temperature (Appendix~\ref{sec:results-tuned}); Appendix~\ref{app:details} documents the configurations and deviations from released code.

\subsection{Deterministic decoding}
\label{sec:results-tempzero}
With the penalty disabled, the ten paths at temperature zero are identical. On GSM8K with 256 steps and blocks of 32, they reproduce the greedy decoder's 73.00 accuracy. Enabling the count penalty at gate 8 ($\dose = 72$) raises coverage to 92.33 and plurality accuracy to 78.60 without token-sampling noise (Table~\ref{tab:argmax}; an example decode in Figure~\ref{fig:latticepair}). This is 0.33 points below the sampled \methodname configuration's five-seed mean of 78.93 (Table~\ref{tab:abl-seeds}). The 128-step results in Table~\ref{tab:nfe} show the same qualitative effect.

\subsection{Temperature and repulsion}
\label{sec:results-mechanism}

Under pure diffusion with 256 steps, we sweep temperature with repulsion disabled, then sweep penalty strength at $\temp = 0.6$. Table~\ref{tab:grids} compares the resulting trade-offs among disagreement, per-sample accuracy, coverage and plurality accuracy. These runs predate the end-of-sequence correction used in Table~\ref{tab:nfe} (Appendix~\ref{app:eosfix}).

\begin{table}[ht]
  \centering
  \scriptsize
  \caption{The two dials swept over their usable range: GSM8K under pure diffusion, $\numpaths{=}10$, 2560 evaluations, $n{=}600$ over two seeds. Panel B's $\dose{=}0$ row is Panel A's $\temp{=}0.6$ row. Higher is better in the accuracy columns; disagreement is a rate, not an objective. Bold marks temperature's best rung on each metric; $\star$ our deployed setting ($\dose = 64$, gate $\gate = 7.11$), which exceeds temperature's best on plurality and coverage (Panel B unmarked; Appendix~\ref{app:grids}).}
  \label{tab:grids}
  \setlength{\tabcolsep}{5pt}
  \renewcommand{\arraystretch}{0.94}
  \makebox[\linewidth]{\hfill
  \begin{tabular}[t]{r!{\color{black!20}\vrule width 0.5pt}rrrr}
    \toprule
    \multicolumn{5}{c}{\emph{A. Temperature, $\dose = 0$}} \\
    \midrule
    $\temp$ & Disagr. & Per-s. & Plur. & Cover. \\
    \midrule
    0.6 & 15.02 & 53.42 & 54.53 & 63.17 \\
    0.8 & 20.87 & \textbf{54.21} & 56.31 & 66.83 \\
    1.5 & 34.86 & 51.95 & 57.83 & 73.08 \\
    2.0 & 50.61 & 47.77 & 60.72 & 79.92 \\
    2.5 & 64.33 & 41.34 & \textbf{62.00} & \textbf{81.25} \\
    3.0 & 88.87 & 16.73 & 35.62 & 62.75 \\
    \bottomrule
  \end{tabular}
  \hspace{3.5em}
  \begin{tabular}[t]{r!{\color{black!20}\vrule width 0.5pt}rrrr}
    \toprule
    \multicolumn{5}{c}{\emph{B. Count penalty, $\temp = 0.6$}} \\
    \midrule
    $\dose$ & Disagr. & Per-s. & Plur. & Cover. \\
    \midrule
    0 & 15.02 & 53.42 & 54.53 & 63.17 \\
    4 & 33.46 & 54.48 & 60.78 & 77.17 \\
    16 & 42.63 & 56.22 & 69.00 & 85.50 \\
    32 & 43.87 & 57.31 & 71.65 & 87.75 \\
    $\star$\,64 & 46.60 & 56.19 & 73.09 & 86.83 \\
    128 & 47.80 & 55.36 & 71.66 & 88.33 \\
    384 & 53.46 & 53.30 & 73.10 & 90.00 \\
    \bottomrule
  \end{tabular}\hfill}
\end{table}

\looseness=-1 Tuning the baseline's own dial does not reach the penalty. The deployed setting exceeds every temperature rung on plurality and coverage at once, 73.09 and 86.83 against column bests of 62.00 and 81.25 (Table~\ref{tab:grids}). Temperature pays for its diversity in per-sample accuracy, which falls across the sweep and drags the vote down past $\temp = 2.5$; the penalty's holds between 53.30 and 57.31. At matched disagreement the per-sample gap plateaus near $+7$ from $\dose = 16$ (deployed $+7.36$), a property of the dial, not one tuned point. The peer term produces it: a control reading the sample's own distribution returns to the temperature curve (Section~\ref{app:abl-conditionality}). In blocks the contrast is $+0.52$.

\subsection{Accuracy at matched compute}
\label{sec:results-main}
\begin{table}[ht]
  \centering
  \scriptsize
  \caption{Every method at 1280 network function evaluations; Table~\ref{tab:main} repeats it at 2560. GSM8K $n{=}600$, MATH $n{=}492$, TruthfulQA $n{=}684$, two to five seeds; TruthfulQA coverage omitted (94.37 chance). Rows run at $\temp{=}0.6$ except the three $\temp{=}0$ rows: row 1 is the greedy decode, equal to self-consistency (Proposition~\ref{prop:degeneracy}), and the two \methodname, count rows use gate 8 ($\dose = 72$). Pure-diffusion rows carry the end-of-sequence correction (Appendix~\ref{app:eosfix}); TAPS ($\dagger$) is uncorrected. Bold: best per column within a regime, disagreement and the $\numpaths{=}1$ greedy row excluded.}
  \label{tab:nfe}
  \setlength{\tabcolsep}{0.9pt}
  \renewcommand{\arraystretch}{0.8}
  \begin{tabular*}{\linewidth}{@{\extracolsep{\fill}}l!{\color{black!20}\vrule width 0.5pt}ccc!{\color{black!20}\vrule width 0.5pt}cc!{\color{black!20}\vrule width 0.5pt}ccc!{\color{black!20}\vrule width 0.5pt}ccc@{}}
    \toprule
    & \multicolumn{3}{c}{Plurality ($\uparrow$)} & \multicolumn{2}{c}{Coverage ($\uparrow$)} & \multicolumn{3}{c}{Per-sample ($\uparrow$)} & \multicolumn{3}{c}{Disagreement} \\
    \cmidrule(lr){2-4} \cmidrule(lr){5-6} \cmidrule(lr){7-9} \cmidrule(lr){10-12}
    Method & \scriptsize GSM & \scriptsize MATH & \scriptsize TQA & \scriptsize GSM & \scriptsize MATH & \scriptsize GSM & \scriptsize MATH & \scriptsize TQA & \scriptsize GSM & \scriptsize MATH & \scriptsize TQA \\
    \midrule
    \multicolumn{12}{l}{\emph{Semi-autoregressive blocks of 32}} \\
    Self-consistency, $\tau{=}0$ ($=$ greedy, $\numpaths{=}1$) & 70.17 & 28.25 & 55.26 & 70.17 & 28.25 & 70.17 & 28.25 & 55.26 & 0.00 & 0.00 & 0.00 \\
    \cmidrule(l{2pt}r{2pt}){1-12}
    Self-consistency & 77.05 & 40.14 & 59.02 & 89.70 & 56.81 & \textbf{70.04} & \textbf{30.02} & \textbf{54.44} & 27.14 & 62.01 & 34.09 \\
    TAPS & 77.69 & 39.47 & 58.27 & 91.00 & 57.32 & 66.77 & 29.26 & 53.53 & 35.51 & 64.47 & 36.63 \\
    ODD ($\alpha{=}64$) & 77.24 & 37.03 & 60.32 & 91.75 & 54.27 & 64.94 & 24.19 & 52.35 & 38.72 & 70.98 & 40.97 \\
    ODD ($\alpha{=}256$) & 78.00 & 37.75 & 59.99 & 92.17 & 54.37 & 66.57 & 22.77 & 49.85 & 36.47 & 68.12 & 42.88 \\
    \cmidrule(l{2pt}r{2pt}){1-12}
    \methodname, count & 79.11 & 40.04 & 59.77 & 93.53 & \textbf{59.86} & 68.21 & 27.67 & 48.07 & 35.53 & 68.82 & 49.88 \\
    \methodname, count, $\tau{=}0$ & 80.38 & \textbf{40.86} & 59.84 & 94.33 & 57.11 & 69.33 & 27.13 & 46.92 & 34.35 & 69.13 & 51.61 \\
    \methodname, exp.\ count & \textbf{80.65} & 39.58 & 59.82 & \textbf{94.67} & 59.65 & 66.47 & 26.78 & 50.75 & 40.60 & 71.45 & 43.66 \\
    \methodname, collision-wt.\ & 80.07 & 38.81 & \textbf{60.60} & 94.42 & 55.39 & 61.87 & 24.44 & 52.78 & 48.63 & 75.08 & 40.53 \\
    \midrule[0.12em]
    \multicolumn{12}{l}{\emph{Pure diffusion (whole response as one block)}} \\
    Self-consistency & 57.02 & 24.00 & 56.27 & 67.00 & 37.40 & 54.47 & 22.87 & 54.17 & 20.40 & 40.89 & 24.81 \\
    TAPS$^\dagger$ & 56.82 & 25.67 & 55.14 & 69.75 & 39.74 & 50.34 & 22.46 & \textbf{55.26} & 32.74 & 45.38 & 14.96 \\
    \cmidrule(l{2pt}r{2pt}){1-12}
    \methodname, count & 71.53 & 29.79 & 58.54 & 87.75 & 50.51 & 53.85 & 22.62 & 53.79 & 49.73 & 70.26 & 36.91 \\
    \methodname, count, $\tau{=}0$ & 70.98 & 31.27 & \textbf{59.70} & 87.17 & 50.41 & 53.82 & 23.17 & 52.85 & 49.50 & 68.40 & 37.57 \\
    \methodname, exp.\ count & 72.01 & 29.40 & 55.65 & \textbf{90.33} & 48.78 & 53.87 & 20.47 & 50.78 & 53.23 & 74.97 & 37.01 \\
    \methodname, collision-wt.\ & \textbf{76.27} & \textbf{35.37} & 56.99 & 90.00 & \textbf{55.79} & \textbf{59.81} & \textbf{25.10} & 50.50 & 45.93 & 72.40 & 41.57 \\
    \bottomrule
  \end{tabular*}
\end{table}

\begin{figure}[t]
  \centering
  \includegraphics[width=0.92\linewidth]{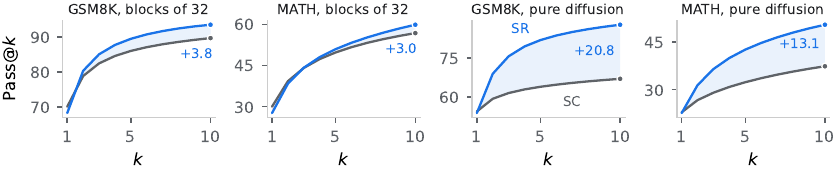}
  \caption{Pass@$k$, self-consistency against \methodname (count statistic), Table~\ref{tab:nfe}'s campaigns; shading is the gap, the printed number its size at $k{=}10$. \methodname is below self-consistency only at $k \le 2$, the per-sample end of the curve, and above it at every larger $k$, including MATH; all methods and values in Table~\ref{tab:passkmain}.}
  \label{fig:passkmain}
\end{figure}

\looseness=-1 At its best variant per column, \methodname tops all six plurality columns of Table~\ref{tab:nfe}, over the strongest baseline by 2.65, 0.72 and 0.28 in blocks; a different variant leads each benchmark. The deployed count variant itself trails self-consistency on MATH in blocks by $0.10$; on GSM8K the strongest baseline is the peer-conditioned ODD, which the temperature-zero count variant alone exceeds ($80.38$ against $78.00$, 128 steps). TruthfulQA is the negative case, nothing to convert at a 94.37 chance ceiling. On MATH the pool improves where no aggregator we built converts it, pass@$k$ reaching $+3.05$ at $k = 10$ (Figure~\ref{fig:passkmain}); the gain is a verifier's, not the vote's (Appendices~\ref{sec:results-decomposition} and~\ref{app:covpred}).

\looseness=-2 Under pure diffusion, where duplication is worst (Appendix~\ref{app:abl-blocks}), self-consistency's GSM8K vote falls to 57.02 against the collision-weighted penalty's 76.27 (Table~\ref{tab:nfe}); TAPS sits at 56.82 and blocking \citep{arriola2025block,nie2025llada} recovers most of the gap, at 77.05.

\subsection{Peer-versus-own-distribution control}
\label{app:abl-conditionality}
\looseness=-1 Reading the peers, not the disagreement the penalty produces, is what keeps per-sample accuracy above the temperature curve. A matched control isolates the two: with everything else fixed, each path reads its own predictive distribution in place of the peers' commitments. At matched disagreement this own-distribution control sits $+0.20$ and $+1.26$ above the temperature sweep, against $+7.35$ and $+12.71$ for the peer statistic (Table~\ref{tab:ladder}).

\begin{table}[h]
  \centering
  \footnotesize
  \caption{Temperature against \methodnamelong under pure diffusion, on the 1200 problem-seed pairs common to all nine configurations. The varied parameter is temperature for the first family and strength $\dose$ for the other two (gate $\gate = \dose/(\numpaths-1)$). The last two columns interpolate the temperature sweep to each configuration's own disagreement and give the per-sample accuracy it reaches there and the gap to it.}
  \label{tab:ladder}
  \setlength{\tabcolsep}{5pt}
  \begin{tabular}{llrrrrrr}
    \toprule
    Conditions on & Varies & Disagr. & Per-sample & Plur. & Cov. & Temp.\ here & Gap \\
    \midrule
    Nothing (temperature)   & $\temp{=}0.6$ & 15.02 & 53.42 & 54.53 & 63.17 & --    & --      \\
    Nothing (temperature)   & $\temp{=}1.0$ & 25.06 & 53.76 & 57.44 & 68.83 & --    & --      \\
    Nothing (temperature)   & $\temp{=}1.2$ & 28.25 & 53.24 & 57.82 & 71.33 & --    & --      \\
    Nothing (temperature)   & $\temp{=}1.5$ & 34.86 & 51.95 & 57.83 & 73.08 & --    & --      \\
    Nothing (temperature)   & $\temp{=}2.0$ & 50.61 & 47.77 & 60.72 & 79.92 & --    & --      \\
    \addlinespace
    Own distribution        & $\dose{=}1$   & 24.48 & 53.94 & 57.57 & 69.33 & 53.74 & $+0.20$ \\
    Own distribution        & $\dose{=}2$   & 32.63 & 53.64 & 59.72 & 74.17 & 52.39 & $+1.26$ \\
    \addlinespace
    Peer comm., count       & $\dose{=}64$  & 46.60 & 56.19 & 73.09 & 86.83 & 48.84 & $+7.35$ \\
    Peer comm., coll.-wt.\  & $\dose{=}128$ & 40.91 & 63.06 & 76.85 & 91.25 & 50.34 & $+12.71$ \\
    \bottomrule
  \end{tabular}
\end{table}

\section{Related work}
\label{sec:related}

\looseness=-1 No common mechanism for diversifying a model's $\numpaths$ samples lets a sample read, mutually and within the step, its peers' committed tokens at the same position. One family reshapes each sample's own distribution, by temperature, truncation, prompt-embedding perturbation \citep{wu2026taps}, tempered remasking \citep{olausson2026twotemps}, or within-sequence kernel-entropy guidance \citep{zhang2026sake}; the draws stay independent. A second couples the pool through shared randomness: arithmetic codes with exact marginals \citep[Prop.~3]{vilnis2023arithmetic} or quasi-Monte Carlo, under a coverage ceiling on marginal-preserving samplers \citep[Thm.~2]{li2026quasimotto}, and sampling without replacement, which changes the sampling probabilities \citep{kool2019stochastic}. A third votes over varied decoding configurations, ensembling block schedules to avoid the failure modes of any single one \citep{lee2025hex}; a fourth reads cross-sample agreement with the opposite sign, keeping agreed positions and resampling the rest \citep{feng2026dvoting}.

\looseness=-1 In beam search, Diverse Beam Search subtracts a Hamming penalty across ordered beam groups \citep{vijayakumar2018diverse} and determinantal beam search maximizes a log-determinant over a similarity kernel \citep{meister2021determinantal}; \methodname instead couples $\numpaths$ samplers through the tokens they commit at each position as they draw. CoT-decoding builds a deterministic autoregressive ensemble \citep{wang2024cotdecoding}, assigning the top-$\numpaths$ first-step tokens to paths and continuing each greedily, so no path reads another. Sequential variants repel each sample from its predecessors, by an entailment score or branch avoidance \citep{park2025diversitysteered,park2026uag}, none voting over the pool; SemDiD instead runs its groups concurrently, repelling each group after the first, greedy one from its most similar other group in embedding space \citep{shi2025semdid}. The closest prior work is ODD \citep{lamont2026odd}, which orthogonalizes each sample against its predecessors in a position-pooled confidence space, one gradient step per denoising step, its feature carrying no position index and its target $\passat@k$. ODD too diversifies at temperature zero, where each sample diverges only from its predecessors; \methodname instead keeps one symmetric penalty across paths, drawing its asymmetry from the information order alone. Kernel-gradient repulsion between concurrent members is established outside text \citep{liu2016svgd,dangelo2021repulsive}, and particle guidance does so in continuous diffusion through a fixed or learned potential \citep{corso2024particle}. Self-repellent random walks \citep{doshi2023srrw} and Stein self-repulsive dynamics \citep{ye2020steinselfrepulsive} carry the name with a different construction, repelling one process from its own past, as does contrastive search, one sequence penalized by its own context \citep{su2022contrastive}.

\looseness=-1 Coupling across samples also appears with the opposite sign: dVoting inherits agreed tokens \citep{feng2026dvoting}, ThinkMerge averages $\numpaths$ traces into one \citep{wang2025thinkmerge}, and others resample toward the trajectory a scorer prefers \citep{dang2025pgdlm,luo2026smc}; D5P4 instead repels by selection: per diffusion step it keeps mutually dissimilar candidates, drawn from the model's denoising logits, under a partitioned determinantal point process \citep{lys2026d5p4,kulesza2012dpp}. A separate line measures the limits of pooling answers without acting on the sampler: vote accuracy can rise then fall in the call count \citep{chen2024more}, one analysis names decorrelation a lever though every intervention it names is peer-blind \citep{bay2026sampling}, and models that err tend to agree on the wrong answer \citep{kim2025correlated}. No label-independent diversity effect exists for the 0-1 loss under any combiner \citep[Thm.~10]{wood2023unified}, and the classical diversity split is additive with no coverage channel \citep{brown2010good}.

\section{Conclusion}
\label{sec:conclusion}

\methodnamelong uses peer commitments to diversify masked-diffusion ensembles without additional model evaluations. Updating these commitments within each denoising step lets paths diverge even at temperature zero. On GSM8K, the count penalty improves plurality over self-consistency by $2.06$ points in blocks of 32 at matched budgets; ten deterministic paths reach $80.38\%$, against $70.17\%$ for the unpenalized greedy decoder (128 steps). The analyses identify higher coverage as the main source of voting gains, while showing that better coverage does not always improve the vote. \textbf{Limitations:} The evidence is limited to the LLaDA family: the main experiments use LLaDA-8B-Instruct, and LLaDA-1.5 covers GSM8K only. The exact allocation result applies locally under shared logits, not to complete trajectories. At fixed penalty strength, plurality changes little beyond $\numpaths = 10$ even as coverage rises (Appendix~\ref{app:abl-k}); selecting the correct answer from the larger pool remains a limitation.

\clearpage
\bibliography{iclr2027_conference}
\bibliographystyle{iclr2027_conference}

\clearpage
\appendix
\numberwithin{proposition}{section}
\section{Proofs and Theoretical Details}
\label{app:proofs}

\subsection{Ensembles without randomness}

\begin{proof}[Proof of Proposition~\ref{prop:degeneracy}]
Write the canvases at the start of a step as $u = (u_1, \dots, u_{\numpaths})$ and the step itself as $u_k \mapsto G(u_k, u_{-k})$, one function $G$ shared by all paths, where $u_{-k}$ is the tuple of path $k$'s peers in a fixed indexing. A statistic that reads the peers as a multiset, or all $\numpaths$ canvases symmetrically, is the special case where $G$ is symmetric in $u_{-k}$.

\emph{One step preserves equality.} Suppose the $u_k$ all equal one canvas $w$. Every path then evaluates $G$ at the same argument $(w, (w, \dots, w))$, and each part of the step returns the same result across paths. The model reads the canvas alone and returns the same logits. The penalty is a function of the same argument and is the same vector. The argmax under a path-index-free tie-break writes the same token. The commit rule, one path-index-free function of the canvas, its raw logits and the written tokens, selects the same positions. So the canvases stay equal after the step.

\emph{Induction.} The paths start from a common initial canvas. One step preserves equality, so induction over steps keeps the canvases identical and the pool holds one distinct sample.
\end{proof}

\begin{remark}[Order-invariance of the cascade]
  \label{prop:canonicity}
  At temperature zero, for a processing order held fixed across steps, $\mathcal{C}_\pi = \mathcal{C}_{\pi'}$ for all processing orders $\pi, \pi'$, where $\mathcal{C}_\pi$ is the pool the cascade produces under order $\pi$. By induction over the pairs (denoising step, cascade position), the path at each pair reads its own canvas and a symmetric count of canvases written at earlier pairs, so the trajectory realized at cascade position $i$ does not depend on the order, which only assigns members to indices; at positive temperature the same induction gives equality in law. The deployed cascade lets each path read every peer; a triangular variant reads only predecessors.
\end{remark}

\subsection{Allocation at a shared position}
\label{app:allocation}

\begin{proof}[Proof of Proposition~\ref{prop:allocation}]
\emph{Concave objective.} Under the hypotheses all paths committing the position face the same logits $\logit$. The $i$-th path to take token $v$ gains $\logit_v - \gate\,(i-1)$, which decreases in $i$. A count vector $\cnt$ then realizes total gain
\[
  \sum_v \cnt_v \logit_v - \gate \sum_v \binom{\cnt_v}{2} ,
\]
a sum over tokens of concave functions of the counts. The cascade takes the largest current gain $\logit_v - \gate\, \cnt_v$ at each turn, so it greedily maximizes this objective.

\emph{Greedy exactness.} A separable concave objective on the integer simplex is maximized greedily \citep{fox1966}. The gains the cascade takes are non-increasing, so every gain it takes is at least $\theta = \max_u (\logit_u - \gate\, \cnt_u)$, the largest gain still available at the end. Take any feasible $\cnt'$ and move one path from a token with $\cnt'_v > \cnt_v$, whose last gain is at most $\theta$, to a token with $\cnt'_u < \cnt_u$, whose next gain is at least $\theta$. The move does not lower the objective. Repeating it until $\cnt' = \cnt$ shows $\cnt$ optimal.

\emph{Water level.} The common threshold is $\theta = \max_u (\logit_u - \gate\, \cnt_u)$. For $\gate > 0$, every token receiving a path has
\[
  (\logit_v - \theta)/\gate \;\le\; \cnt_v \;\le\; (\logit_v - \theta)/\gate + 1 ,
\]
while every token receiving none has $\logit_v \le \theta$. The expected-count and collision-weighted statistics charge tokens no peer has committed, so this allocation does not describe them.
\end{proof}

\subsection{Allocation as a sparsemax}
\label{app:sparsemax}

\emph{Fractional maximizer.} On the constraint $\sum_v \cnt_v = \numpaths$ we have $\sum_v \binom{\cnt_v}{2} = \tfrac12\|\cnt\|^2 - \tfrac{\numpaths}{2}$, so the objective of Proposition~\ref{prop:allocation} for the count statistic is
\[
  -\tfrac{\gate}{2}\,\|\cnt - \logit/\gate\|^2
\]
plus terms independent of $\cnt$. Over fractional allocations its maximizer is the Euclidean projection of $\logit/\gate$ onto the scaled simplex. After the substitution $\cnt = \numpaths p$ this projection is $\numpaths\,\mathrm{sparsemax}(\logit/\gate\numpaths)$ \citep{martins2016sparsemax}, with $\gate\numpaths = \dose\numpaths/(\numpaths-1)$ as the sparsemax temperature.

\emph{Integer counts.} Let $\theta = \max_u(\logit_u - \gate\cnt_u)$ be the cascade's final level and set $y_v(t) = \max(0,\ (\logit_v - t)/\gate + \tfrac12)$. The end-state conditions of the proof above read $|\cnt_v - y_v(\theta)| \le \tfrac12$ for every $v$, while the projection is $y_v(\theta')$ at the level $\theta'$ where $\sum_v y_v(\theta') = \numpaths$. Every positive $y_v$ has slope $-1/\gate$ in $t$, so a level difference above $\gate/2$ would carry $\sum_v y_v$ strictly past $\numpaths$. Hence $|\theta - \theta'| \le \gate/2$, and
\[
  \max_v\, \big|\cnt_v - \numpaths\,\mathrm{sparsemax}_v(\logit/\gate\numpaths)\big| \;<\; 1 ,
\]
the equality excluded because $|\theta - \theta'| = \gate/2$ forces the two allocations to coincide. Figure~\ref{fig:cascade-simplex} shows the six commits as a walk on the probability simplex toward the sparsemax point.

\begin{figure}[t]
  \centering
  \includegraphics[width=\linewidth]{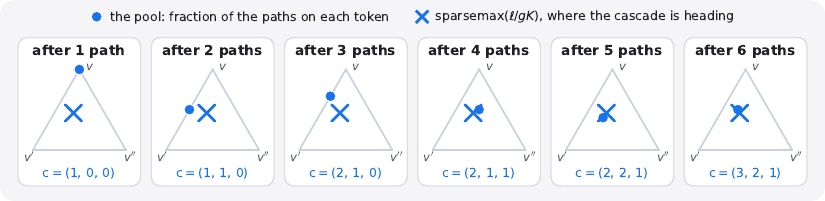}
  \caption{One position on the probability simplex, three of its tokens. Each frame shows the pool after $i$ paths, the fraction of the paths on each token, as a point of the triangle, against $\mathrm{sparsemax}(\logit/\gate\numpaths)$, marked by a cross. The cascade starts the pool at the mode's vertex and walks it, one path at a time, onto the cross: after six paths the pool is $(3,2,1)/6$, the nearest point of the $\numpaths$-grid to $\mathrm{sparsemax}(\logit/\gate\numpaths) = (0.46, 0.34, 0.21)$ (Proposition~\ref{prop:allocation}).}
  \label{fig:cascade-simplex}
\end{figure}

\begin{figure}[t]
  \centering
  \includegraphics[width=\linewidth]{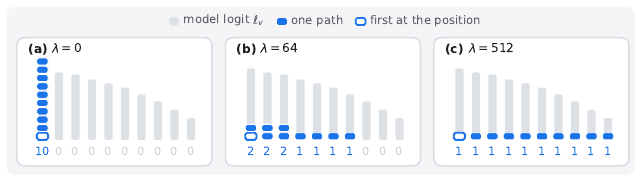}
  \caption{What the $\numpaths$ paths commit at one position at $\temp = 0$: with the penalty off \textbf{(a)}, at the deployed strength \textbf{(b)} and at a large one \textbf{(c)}. Each grey bar is a token's model logit $\logit_v$ and each blue cell is one path committed there, stacked in commit order; the outlined cell is the first path to commit, which faces no peer and pays no penalty. With the penalty off all $\numpaths$ paths tower on the mode; raising it spreads them, until at large strength each of the top $\numpaths$ tokens holds one path. Logits are illustrative and the allocation is computed by the update rule; the water level is the common threshold a token's logit must clear to receive a path.}
  \label{fig:waterlevel}
\end{figure}

\subsection{Tilt}
\label{app:tilt}

Subtracting $\efftilt \divfn_v$ from the log-probabilities multiplies $p(v)$ by $e^{-\efftilt \divfn_v}$ and renormalizes. Where the mass goes is immediate. The ratio is
\[
  \frac{q(v)}{p(v)} \;=\; \frac{e^{-\efftilt \divfn_v}}{\partf} , \qquad \partf = \E_p\big[e^{-\efftilt \divfn}\big] ,
\]
an average of the factors over $\mathrm{supp}\, p$, so for $\efftilt > 0$ a token gains mass, $q(v) \ge p(v)$, exactly when $\divfn_v \le -\log \partf / \efftilt$. For the count statistic under a full-support $p$, once some peer has committed and the support is larger than $\numpaths - 1$, $\min \divfn = 0 < \max \divfn$ and $\partf$ lies strictly between the extreme factors. So every untaken token gains mass and every token of the largest count loses it, the movement Figure~\ref{fig:mechanism} draws against temperature, which reshapes every token instead.

\paragraph{Ratio preservation.}
If $\divfn_u = \divfn_v$ the tilt factors cancel in the ratio $q(u)/q(v)$, which stays $p(u)/p(v)$ at every $\dose$. Temperature instead raises every ratio to the power $\temp/\temp'$ and so distorts all but the unit ratio.

\subsection{Minimum-divergence characterization of the tilt}

\begin{lemma}[I-projection]
  \label{lem:iprojection}
  For $\efftilt \ge 0$ and $m = \E_q[\divfn]$, the law in \eqref{eq:tilt} is the unique solution of $\min\{\kl{q'}{p} : \E_{q'}[\divfn] \le m\}$.
\end{lemma}

\begin{proof}[Proof of Lemma~\ref{lem:iprojection}]
For any law $q'$ on $\vocab$, expand $q$ inside the divergence:
\[
  \kl{q'}{q} \;=\; \sum_v q'(v) \log \frac{q'(v)}{p(v)\,e^{-\efftilt \divfn_v}/\partf} \;=\; \kl{q'}{p} \;+\; \efftilt\,\E_{q'}[\divfn] \;+\; \log \partf .
\]
Rearranged, this reads $\kl{q'}{p} + \efftilt\,\E_{q'}[\divfn] \ge -\log \partf$, with equality if and only if $q' = q$. The tilt $q$ is feasible, since $\E_q[\divfn] = m$. For any feasible $q'$, using $\efftilt \ge 0$ at the second step,
\[
  \kl{q'}{p} \;\ge\; -\log \partf - \efftilt\,\E_{q'}[\divfn] \;\ge\; -\log \partf - \efftilt m \;=\; \kl{q}{p} ,
\]
the last equality being the line above evaluated at $q' = q$. Equality throughout forces $\kl{q'}{q} = 0$, hence $q' = q$. The minimizer therefore exists and is unique.
\end{proof}

The constraint is one-sided because $\efftilt \ge 0$: the tilt is closest to $p$ among laws with expected agreement \emph{at most} $m$, a negative strength bounding the agreement below instead. Fixing $\E[\divfn]$ exactly leaves the same minimizer, the I-projection onto a linear family.

\subsection{Coupling ceiling on coverage}
\label{app:frechet}

\begin{proof}[Proof of Proposition~\ref{prop:frechet}]
Fix the per-sample marginals $p_1, \dots, p_{\numpaths}$ for the event that a sample produces the correct answer. Coverage over all couplings lies in $[\max_k p_k,\ \min(1, \sum_k p_k)]$, the Fr\'echet--Hoeffding bounds \citep[Thms.~2.10.12--13]{nelsen2006copulas} applied to the complement events, while independence gives $1 - \prod_k (1 - p_k)$. The excess of the upper bound over independence is
\[
  G(p) \;=\; \min\Big(1, \textstyle\sum_k p_k\Big) \;-\; \Big[1 - \prod_k (1 - p_k)\Big] .
\]
Write $s = \sum_k p_k$. For $s \le 1$ the arithmetic-geometric mean inequality gives $G \le s - 1 + (1 - s/\numpaths)^{\numpaths}$, whose derivative $1 - (1 - s/\numpaths)^{\numpaths - 1}$ is non-negative, so the maximum is at $s = 1$. For $s > 1$, $G = \prod_k (1 - p_k) \le (1 - s/\numpaths)^{\numpaths} < (1 - 1/\numpaths)^{\numpaths}$. Either way $G \le (1 - 1/\numpaths)^{\numpaths}$, attained only at $p_k \equiv 1/\numpaths$. The value is below $e^{-1}$ for every finite $\numpaths$ and equals $0.3487$ at $\numpaths = 10$; the per-problem excess $(1 - p)^{\numpaths}$ is of order $10^{-3}$ at $p = 0.5$.
\end{proof}

\section{Experimental details}
\label{app:details}

\subsection{Metrics}
\label{app:metrics}
Each metric is computed per problem, averaged over seeds then problems, from the answers $a_1, \dots, a_{\numpaths}$ extracted from the $\numpaths$ samples, where extraction may fail. \emph{Plurality accuracy} scores the vote under a uniform tie-break, $\mathbf{1}[a^\star \in M]/|M|$ for $M$ the top-count set among parsed answers, rather than an order-dependent break. \emph{Coverage} is $\mathbf{1}[\exists k : a_k = a^\star]$, an upper bound on any aggregator over the pool; it saturates on small answer sets (four options give $1 - 0.75^{10} = 94.37\%$ by chance), so we omit it on TruthfulQA, where the count row's coverage sits at that ceiling (Table~\ref{tab:decomp}). \emph{Per-sample accuracy} is the fraction of correct samples, an unparsed sample counted wrong; the unparsed rate reaches 17.4\% on MATH and the two orderings disagree there under pure diffusion, so this counted-wrong denominator, not the parsed-only one, is the one the tables use. \emph{Disagreement} is the rate at which two samples of a pool give different parsed answers, a diversity measure and not an objective.

\subsection{Model and computation}
We run LLaDA-8B-Instruct (\texttt{GSAI-ML/LLaDA-8B-Instruct}) with frozen weights in BF16, $\numpaths = 10$ samples of generation length 256, so a problem costs $\numpaths T$ network function evaluations and no method uses more. Table~\ref{tab:nfe} runs $T = 128$ (1280 evaluations); Tables~\ref{tab:grids}, \ref{tab:main}, \ref{tab:argmax} and~\ref{tab:filtration} run $T = 256$ (2560); the sequential variant spends 1280. The temperature-zero rows of Tables~\ref{tab:nfe} and~\ref{tab:tuned} run at $T = 128$, those of Tables~\ref{tab:argmax} and~\ref{tab:filtration} at $T = 256$. Decoding is semi-autoregressive in blocks of 32 with the released low-confidence remasking, positions ranked by the untilted confidence of the sampled token; the pure-diffusion campaigns instead decode the whole generation as a single block. The pure-diffusion rows of Table~\ref{tab:nfe} carry the released end-of-sequence confidence correction (Appendix~\ref{app:eosfix}), while the 2560-evaluation campaigns of Tables~\ref{tab:grids} and~\ref{tab:main} predate it and every TAPS row cannot take it, as their captions state. Our temperature 0.6 and length 256 deviate from the released greedy decoder, which uses length and step count 512.

\subsection{Selection, seeds and reporting}
GSM8K \citep{cobbe2021training} uses a frozen 600-problem index into the GSM8K test split (\texttt{test.jsonl}, 1319 items), MATH \citep{hendrycks2021math} a 492-problem index into MATH-500 with the eight items whose reference answers do not parse removed, and TruthfulQA \citep{lin2022truthfulqa} a 684-problem index into the four-option items of the multiple-choice set (\texttt{truthful\_qa\_mc}), each item tagged by its SHA-256 prefix in every output. GSM8K is prompted zero-shot with a step-by-step instruction and the last number in the generation is taken as the answer; MATH and TruthfulQA place the final answer in a \texttt{\textbackslash boxed\{\}} field inside a reasoning-then-answer template, extracted as the boxed expression under Hendrycks normalization (MATH) or the boxed option letter (TruthfulQA), each with a bare-answer fallback. Every within-campaign contrast is paired within one decoding run and not pooled across runs, which removes the half-precision kernel noise that compounds over the denoising steps; the few cross-campaign comparisons (the end-of-sequence correction of Appendix~\ref{app:eosfix}, the block-versus-canvas baselines) are marked as such where they appear. Configurations carry two to five seeds as stated per table; temperature-zero configurations consume no randomness and have no seed axis, so a paired sampled configuration carries the seed variance. A reported result is a mean over at least 300 problems and two seeds; configurations at 200 problems or one seed are given in direction only and carry no table. The per-stratum diagnostics of Figure~\ref{fig:covpred} report finer strata, some below this floor, as diagnostics and not results. A two-configuration comparison is a mean paired difference over the shared problems with seeds pooled; Appendix~\ref{app:abl-seeds} gives the single-configuration seed spread.

\subsection{Baselines}
All baseline numbers are our reproductions under this protocol, each family swept over its own temperature, ties broken uniformly. \emph{ODD} and \emph{TAPS} run from their authors' released code at pinned commits.\footnote{ODD (\url{https://github.com/sean-lamont/odd}) and TAPS (\url{https://github.com/Johnny221B/TAPS}), each vendored at a pinned commit.} ODD runs its own strategy and feature extractor inside our loop, at temperature 0.6 in the shared-schedule tables (Tables~\ref{tab:nfe} and~\ref{tab:main}) and at 1.0, a temperature its authors report, in the own-temperature comparison (Table~\ref{tab:tuned}), with step sizes $\alpha \in \{64, 256\}$ from our sweep; TAPS runs its embedding-level conditioning perturbation, cosine-annealed Gaussian noise of scale 0.2 on the prompt embeddings, swept over $\{0.6, 1.0, 1.5\}$. Sampling hygiene is left as each author wrote it, so an unparsed sample is counted wrong rather than repaired.

\section{Ablations}
\label{app:ablations}

We vary the decoding parameters to assess the method's sensitivity. Ensemble size and decoding regime have the largest effects on plurality accuracy; the remaining parameters have smaller effects over the ranges tested. Unless stated otherwise, each sweep varies one parameter while holding the others at the reference configuration of Appendix~\ref{app:details}: ten samples, 128 steps, blocks of 32, $\temp = 0.6$, and the count statistic at strength 64, evaluated on GSM8K ($n = 600$) at 1280 model evaluations. Comparisons use the problems common to every row and seed, with metrics defined in Appendix~\ref{app:metrics}. The ensemble-size sweep (Table~\ref{tab:abl-k}) and the penalty-by-regime comparison (Table~\ref{tab:crossover}) use 256 steps instead. The ensemble-size sweep also includes a fixed-gate comparison at $\gate = 8.00$, alongside the fixed-strength comparison at $\dose = 64$.

\subsection{Variation across decoding seeds}
\label{app:abl-seeds}
We measure variation across five decoding seeds while holding the evaluation problems and decoding configuration fixed. At 1280 evaluations, \methodname has mean plurality accuracy 79.11\%, with a standard deviation of 1.08 percentage points and a range of 2.87 points. At 2560 evaluations, its standard deviation and range are 0.23 and 0.50 points, respectively. Self-consistency at 1280 evaluations has a standard deviation of 0.40 points and a range of 1.04 points. In these runs, reducing the step count from 256 to 128 increases \methodname's standard deviation by a factor of 4.7 and its range by a factor of 5.7. Per-sample accuracy varies less across seeds than plurality accuracy in every configuration.

For the comparison with self-consistency, we report the paired difference for each seed at the same evaluation budget. At 1280 evaluations, these differences are 3.51, 1.68, 1.74, 1.68, and 1.69 percentage points, giving a mean improvement of 2.06 points. The improvement is positive for all five seeds. These summaries describe sensitivity to sampling randomness on the fixed evaluation set; the seed standard deviations and ranges are not significance thresholds.

\begin{table}[h]
  \centering
  \footnotesize
  \caption{Effect of the decode seed. Each configuration is held fixed, block 32, GSM8K, $n=600$; only the seed varies. The reference configuration and self-consistency are at 1280 evaluations; the third configuration is the same method at 2560. SD and range are taken over the five seeds.}
  \label{tab:abl-seeds}
  \setlength{\tabcolsep}{4pt}
  \begin{tabular}{llrrrrrrrr}
    \toprule
    Configuration & Metric & Seed 0 & Seed 1 & Seed 2 & Seed 3 & Seed 4 & Mean & SD & Range \\
    \midrule
    \methodname, 1280        & Plurality   & 80.93 & 78.89 & 78.79 & 78.06 & 78.89 & 79.11 & 1.08 & 2.87 \\
                             & Coverage    & 94.33 & 92.17 & 94.00 & 93.50 & 93.67 & 93.53 & 0.83 & 2.17 \\
                             & Per-sample  & 67.68 & 68.12 & 68.72 & 68.00 & 68.52 & 68.21 & 0.41 & 1.03 \\
    \addlinespace
    Self-consistency, 1280   & Plurality   & 77.42 & 77.21 & 77.06 & 76.38 & 77.20 & 77.05 & 0.40 & 1.04 \\
                             & Coverage    & 89.50 & 89.67 & 90.67 & 89.50 & 89.17 & 89.70 & 0.57 & 1.50 \\
                             & Per-sample  & 70.03 & 69.65 & 70.57 & 70.02 & 69.92 & 70.04 & 0.33 & 0.92 \\
    \addlinespace
    \methodname, 2560        & Plurality   & 78.82 & 79.22 & 78.76 & 78.72 & 79.14 & 78.93 & 0.23 & 0.50 \\
                             & Coverage    & 94.33 & 93.33 & 94.00 & 92.00 & 93.00 & 93.33 & 0.91 & 2.33 \\
                             & Per-sample  & 71.13 & 70.73 & 70.53 & 70.82 & 70.77 & 70.80 & 0.22 & 0.60 \\
    \bottomrule
  \end{tabular}
\end{table}

\subsection{Main comparison at 2560 evaluations}
\label{app:mainfull}
Table~\ref{tab:main} repeats the temperature-0.6 rows of Table~\ref{tab:nfe} at twice the computation, method by method, in both regimes; pure diffusion is the mechanism's measurement ground, not the operating point (Appendix~\ref{app:abl-blocks}).

\begin{table}[h]
  \centering
  \footnotesize
  \caption{Both decoding regimes under the protocol of Section~\ref{sec:experiments} ($K = 10$, 2560 NFEs, temperature 0.6, two to five seeds per configuration, five on the GSM8K \methodname, count row and two or three on the rest); TruthfulQA coverage is unreported, since guessing over four options already reaches 94.37. Bold: best per column within a regime, disagreement and the $K{=}1$ greedy reference row excluded. The pure-diffusion rows predate the end-of-sequence correction applied in Table~\ref{tab:nfe} (Appendix~\ref{app:eosfix}). Baselines are our reproductions (Appendix~\ref{app:details}).}
  \label{tab:main}
  \setlength{\tabcolsep}{2.0pt}
  \begin{tabular*}{\linewidth}{@{\extracolsep{\fill}}l!{\color{black!20}\vrule width 0.5pt}ccc!{\color{black!20}\vrule width 0.5pt}cc!{\color{black!20}\vrule width 0.5pt}ccc!{\color{black!20}\vrule width 0.5pt}ccc@{}}
    \toprule
    & \multicolumn{3}{c}{Plurality ($\uparrow$)} & \multicolumn{2}{c}{Coverage ($\uparrow$)} & \multicolumn{3}{c}{Per-sample ($\uparrow$)} & \multicolumn{3}{c}{Disagreement} \\
    \cmidrule(lr){2-4} \cmidrule(lr){5-6} \cmidrule(lr){7-9} \cmidrule(lr){10-12}
    Method & \scriptsize GSM & \scriptsize MATH & \scriptsize TQA & \scriptsize GSM & \scriptsize MATH & \scriptsize GSM & \scriptsize MATH & \scriptsize TQA & \scriptsize GSM & \scriptsize MATH & \scriptsize TQA \\
    \midrule
    \multicolumn{12}{l}{\emph{Semi-autoregressive blocks of 32}} \\
    Single greedy decode ($K{=}1$) & 73.00 & -- & -- & 73.00 & -- & 73.00 & -- & -- & 0.00 & -- & -- \\
    \cmidrule(l{2pt}r{2pt}){1-12}
    Self-consistency & 76.09 & 38.18 & 59.84 & 87.94 & 55.89 & \textbf{72.31} & \textbf{31.60} & 56.37 & 19.93 & 55.39 & 28.47 \\
    TAPS & 75.37 & 37.54 & 58.46 & 90.17 & 55.39 & 68.88 & 30.26 & 55.42 & 29.36 & 57.89 & 30.80 \\
    ODD ($\alpha{=}64$) & 77.75 & 37.49 & 59.31 & 91.58 & 56.20 & 69.14 & 26.88 & 54.01 & 31.00 & 66.43 & 37.55 \\
    ODD ($\alpha{=}256$) & 78.24 & 39.37 & 59.93 & 92.33 & 54.78 & 69.52 & 27.71 & 54.20 & 31.10 & 63.11 & 38.05 \\
    \cmidrule(l{2pt}r{2pt}){1-12}
    \methodname, count & 78.93 & 40.08 & 61.33 & 93.33 & \textbf{60.47} & 70.80 & 30.22 & \textbf{56.53} & 30.85 & 64.25 & 34.17 \\
    \methodname, exp.\ count & \textbf{79.89} & 39.79 & 60.01 & 94.00 & 58.43 & 69.81 & 29.77 & 52.64 & 33.49 & 64.04 & 40.76 \\
    \methodname, collision-wt.\ & 79.20 & \textbf{40.16} & \textbf{61.99} & \textbf{94.67} & 58.74 & 67.01 & 28.25 & 54.62 & 38.75 & 67.49 & 38.14 \\
    \midrule[0.12em]
    \multicolumn{12}{l}{\emph{Pure diffusion (whole response as one block)}} \\
    Self-consistency & 54.53 & 25.98 & 56.14 & 63.17 & 36.08 & 53.42 & 23.77 & 50.02 & 15.02 & 34.27 & 22.20 \\
    TAPS & 56.77 & 27.13 & 56.20 & 69.92 & 40.96 & 52.91 & 23.69 & \textbf{55.87} & 28.69 & 41.91 & 13.02 \\
    \cmidrule(l{2pt}r{2pt}){1-12}
    \methodname, count & 73.09 & 31.58 & 56.79 & 86.83 & 51.52 & 56.19 & 23.37 & 43.65 & 46.60 & 65.23 & 56.76 \\
    \methodname, exp.\ count & \textbf{76.87} & 34.44 & 56.60 & 90.67 & 53.96 & 62.52 & 25.10 & 51.46 & 40.15 & 64.06 & 37.88 \\
    \methodname, collision-wt.\ & 76.85 & \textbf{37.85} & \textbf{58.74} & \textbf{91.25} & \textbf{55.79} & \textbf{63.06} & \textbf{27.79} & 52.57 & 40.91 & 66.64 & 38.39 \\
    \bottomrule
  \end{tabular*}
\end{table}

\subsection{Ensemble size}
\label{app:abl-k}
Ensemble size is the parameter with the largest effect. The sample count enters the budget directly and the penalty through its normalization, so a sweep over it is not budget-matched; the strength moves with it.

The vote rises to ten samples and is then flat at fixed $\dose$: 78.93, 79.40, 79.38 at $\numpaths = 10, 20, 32$, while coverage keeps rising, 93.33, 95.50, 96.33, and per-sample accuracy holds, 70.80 to 71.04. At fixed gate the vote reads 80.29 and 81.18 at $\numpaths = 20, 32$, but per-sample accuracy falls 70.80 to 67.19 to 61.86 and disagreement climbs 30.85 to 48.25, the temperature signature of Table~\ref{tab:grids}; that rise is more spreading and cannot be read as $\numpaths$-scaling. Self-consistency's vote is flat from ten onward while its coverage rises, so the vote is the pool's bottleneck, not the penalty's. Below ten the vote climbs, 71.47 to 78.93 from two samples to ten and coverage 79.72 to 93.33; five samples at 1280 evaluations read 76.88, above self-consistency at ten and twice the budget. We hold $\dose$ across $\numpaths$.

\begin{table}[h]
  \centering
  \small
  \caption{Effect of ensemble size ($\numpaths$). Fixed: 256 steps, block 32, temperature 0.6, count statistic, GSM8K, $n=600$. Not budget-matched: NFE $= \numpaths \times$ steps and is the price of each row. Rows below ten hold the gate at 8.00; the two conventions meet at $\numpaths = 9$ ($\dose = 64$). Rows above ten hold $\dose = 64$, the $\numpaths$-invariant scaling of Corollary~\ref{cor:sparsemax}, with the fixed-gate rows as the diagnostic; self-consistency rows above ten pool seeds, which their i.i.d.\ draws license. The $\dose = 0$ rows are self-consistency, the same loop with the penalty off; a dash is not measured.}
  \label{tab:abl-k}
  \setlength{\tabcolsep}{5pt}
  \begin{tabular}{rrrrrrrrrr}
    \toprule
    $\numpaths$ & $\dose$ & $\gate$ & NFE & $n$ & Seeds & Plurality & Coverage & Per-sample & Disagr. \\
    \midrule
    \multicolumn{10}{l}{\emph{Fixed gate $\gate = 8.00$}} \\
    2  & 8   & 8.00 & 512   & 600 & 3 & 71.47 & 79.72 & 71.47 & 25.39 \\
    3  & 16  & 8.00 & 768   & 600 & 3 & 74.44 & 84.39 & 72.00 & 26.07 \\
    5  & 32  & 8.00 & 1280  & 600 & 3 & 76.88 & 88.78 & 71.39 & 28.03 \\
    7  & 48  & 8.00 & 1792  & 600 & 3 & 77.57 & 91.50 & 71.48 & 28.55 \\
    \addlinespace
    \multicolumn{10}{l}{\emph{Fixed strength $\dose = 64$, the $\numpaths$-invariant scaling of Corollary~\ref{cor:sparsemax}}} \\
    10 & 64  & 7.11 & 2560  & 600 & 5 & 78.93 & 93.33 & 70.80 & 30.85 \\
    20 & 64  & 3.37 & 5120  & 600 & 2 & 79.40 & 95.50 & 70.93 & 29.78 \\
    32 & 64  & 2.06 & 8192  & 600 & 2 & 79.38 & 96.33 & 71.04 & 29.50 \\
    \addlinespace
    \multicolumn{10}{l}{\emph{Fixed gate $\gate = 7.11$ (diagnostic; per-sample falls, disagreement climbs)}} \\
    20 & 135 & 7.11 & 5120  & 600 & 2 & 80.29 & 96.58 & 67.19 & 38.48 \\
    32 & 220 & 7.11 & 8192  & 600 & 2 & 81.18 & 97.75 & 61.86 & 48.25 \\
    \addlinespace
    \multicolumn{10}{l}{\emph{Self-consistency ($\dose = 0$); rows above ten pool seeds}} \\
    10 & 0   & 0.00 & 2560  & 600 & 3 & 76.09 & 87.94 & 72.31 & 19.93 \\
    20 & 0   & 0.00 & 5120  & 600 & 5 & 76.37 & 89.81 & --    & --    \\
    32 & 0   & 0.00 & 8192  & 600 & 5 & 76.48 & 90.85 & --    & --    \\
    50 & 0   & 0.00 & 12800 & 600 & 5 & 76.83 & 91.67 & --    & --    \\
    \bottomrule
  \end{tabular}
\end{table}

\subsection{Small-effect parameters}
\label{app:abl-sweeps}
\label{app:abl-steps}
The remaining parameters move the vote least; Table~\ref{tab:abl-sweeps} collects the two closest, the samples-versus-steps split of a fixed budget and the choice of statistic, as deviations from the shared reference. Sample count dominates step count at a fixed 1280-evaluation budget: ten samples of 128 steps read 79.11 against 76.88 for five of 256, coverage 93.53 against 88.78, per-sample the other way. The two distribution-reading variants replace the count by the peers' predicted distributions $\pbar[j,\cdot]$: \emph{expected-count} uses the expected number of peers on each token under $\pbar$, and \emph{collision-weighted} scales that by the collision probability between the path's own law and the peers' mean. Among the statistics, expected-count reads highest at 80.65 and collision-weighted 80.07 against the count's 79.11; those gaps of 1.54 and 0.96 reverse across benchmarks, so we deploy the count everywhere.

\begin{table}[h]
  \centering
  \small
  \caption{The small-effect parameters, as deviations from the reference configuration (GSM8K, $n=600$, $\numpaths=10$, 128 steps, block 32, $\temp=0.6$, count statistic, 1280 network function evaluations). The reference and self-consistency rows are shared by both blocks; the statistic rows move only the statistic, while the split row re-allocates the fixed budget to five samples of 256 steps and so also lowers the strength to $\dose=32$ (gate 8.00). Seeds as stated; higher is better in the accuracy columns; disagreement is a rate, not an objective.}
  \label{tab:abl-sweeps}
  \setlength{\tabcolsep}{5pt}
  \begin{tabular}{llrrrrr}
    \toprule
    Parameter & Setting & Seeds & Plurality & Coverage & Per-sample & Disagr. \\
    \midrule
    Reference        & count, $\dose{=}64$, $\gate{=}7.11$ & 5 & 79.11 & 93.53 & 68.21 & 35.53 \\
    Self-consistency & $\dose{=}0$, $\gate{=}0$  & 5 & 77.05 & 89.70 & 70.04 & 27.14 \\
    \addlinespace
    Split at 1280 NFE & 5 samples $\times$ 256 steps, $\dose{=}32$ & 3 & 76.88 & 88.78 & 71.39 & 28.03 \\
    \addlinespace
    Statistic        & expected count          & 2 & 80.65 & 94.67 & 66.47 & 40.60 \\
    Statistic        & collision-weighted& 2 & 80.07 & 94.42 & 61.87 & 48.63 \\
    \bottomrule
  \end{tabular}
\end{table}

The vote is flat in the denoising step count, which carries no table of its own: doubling the step budget from 1280 to 2560 evaluations moves the vote 0.18 over five seeds against the 2.06 the method gains over self-consistency.

\subsection{Penalty by decoding regime}
\label{app:abl-crossover}
On per-sample accuracy the ordering of the count and collision-weighted statistics reverses between the two regimes on all three benchmarks (Table~\ref{tab:crossover}), so neither distribution-aware statistic transfers between regimes.

\begin{table}[h]
  \centering
  \small
  \caption{Penalty by decoding regime on per-sample accuracy, paired per problem and seed-averaged over seeds $\{0,1\}$, at 2560 network function evaluations. Each entry is the named penalty minus the count statistic in that regime, positive favouring the named penalty; the interaction is the pure-diffusion effect minus the block-32 effect, taken within problem. These within-problem paired contrasts on two seeds do not reproduce the seed-mean differences of Table~\ref{tab:main}.}
  \label{tab:crossover}
  \setlength{\tabcolsep}{4.5pt}
  \begin{tabular}{lrrrrrr}
    \toprule
    & \multicolumn{3}{c}{Expected count $-$ count} & \multicolumn{3}{c}{Collision-weighted $-$ count} \\
    \cmidrule(lr){2-4} \cmidrule(lr){5-7}
    Benchmark & Diffusion & Block 32 & Interaction & Diffusion & Block 32 & Interaction \\
    \midrule
    GSM8K ($n=600$) & $+6.33$ & $-1.53$ & $+7.86$  & $+6.87$ & $-4.23$ & $+11.10$ \\
    MATH ($n=492$)  & $+1.73$ & $-0.46$ & $+2.18$  & $+4.42$ & $-1.97$ & $+6.39$  \\
    TruthfulQA ($n=684$) & $+7.81$ & $-3.70$ & $+11.51$ & $+8.91$ & $-1.45$ & $+10.36$ \\
    \bottomrule
  \end{tabular}
\end{table}

\subsection{Decoding regime}
\label{app:abl-blocks}
Pure diffusion decodes the whole response as one block and is LLaDA's default decoding; for the Instruct model, LLaDA reports blocks of 32 as stronger on GSM8K and MATH \citep{nie2025llada}. The penalty reads $+2.06$ over self-consistency in blocks and $+15.91$ under pure diffusion, as differences of the marginals below (these pure-diffusion rows predate the end-of-sequence correction; with it the gain is $+14.50$, Table~\ref{tab:nfe}); the second regime is where the pool is most redundant, so it is where the mechanism has the most to convert.

\begin{table}[h]
  \centering
  \small
  \caption{Effect of the decoding regime. Fixed: $\numpaths = 10$, 128 steps, temperature 0.6, count statistic, GSM8K, $n=600$. Both pure-diffusion rows predate the released end-of-sequence correction; Table~\ref{tab:nfe} carries the corrected rows and Appendix~\ref{app:eosfix} reports what the correction is worth. Per-sample accuracy under the parsed-answer denominator is 0.41 higher on both pure-diffusion rows.}
  \label{tab:abl-blocks}
  \setlength{\tabcolsep}{5pt}
  \begin{tabular}{lrrrrrrrr}
    \toprule
    Regime & $\dose$ & $\gate$ & $n$ & Seeds & Plurality & Coverage & Per-sample & Disagr. \\
    \midrule
    Blocks of 32   & 64 & 7.11 & 600 & 5 & 79.11 & 93.53 & 68.21 & 35.53 \\
    Blocks of 32   & 0  & 0.00 & 600 & 5 & 77.05 & 89.70 & 70.04 & 27.14 \\
    \addlinespace
    Pure diffusion & 64 & 7.11 & 600 & 2 & 70.30 & 87.25 & 52.58 & 51.18 \\
    Pure diffusion & 0  & 0.00 & 600 & 2 & 54.39 & 65.83 & 52.02 & 21.08 \\
    \bottomrule
  \end{tabular}
\end{table}

\subsection{Recommended defaults}
\label{app:abl-defaults}
\begin{table}[H]
  \centering
  \small
  \caption{Recommended settings from the sweeps above. The range column states where the setting is supported by a measurement; a dash means we have no range.}
  \label{tab:abl-defaults}
  \footnotesize
  \setlength{\tabcolsep}{5pt}
  \begin{tabular}{lll>{\raggedright\arraybackslash}p{6cm}}
    \toprule
    Parameter & Recommended & Range & Note \\
    \midrule
    Ensemble size $\numpaths$   & 10           & 2 to 32    & Vote flat past ten at fixed $\dose$; coverage still rising \\
    Strength $\dose$            & 64           & --         & Selected on 200 problems at one seed \\
    Strength across $\numpaths$ & $\dose = 64$ & 10 to 32   & Hold $\dose$, not $\gate$ (Corollary~\ref{cor:sparsemax}); fixed $\gate$ at $\numpaths \ge 20$ costs per-sample accuracy \\
    Denoising steps             & 128          & 128 to 256 & Vote flat in the step count; 128 is cheaper \\
    Temperature                 & 0.6 or 0     & 0 to 0.6   & Zero (gate 8, $\dose{=}72$) is 1.27 above the sampled configuration ($\dose{=}64$) at 128 steps and 0.33 below it at 256; uses no randomness \\
    Statistic                   & Count        & --         & Ordering is benchmark-dependent \\
    Decoding regime             & Blocks of 32 & --         & Stronger than pure diffusion on GSM8K and MATH \citep{nie2025llada} \\
    \bottomrule
  \end{tabular}
\end{table}

\subsection{Second checkpoint: LLaDA-1.5}
\label{app:llada15}
\looseness=-1 We repeat the deployed configuration on LLaDA-1.5 \citep{zhu2025llada15}, the preference-optimized checkpoint of LLaDA-8B-Instruct, which the released decoder runs unchanged; the paired gains over self-consistency are in Table~\ref{tab:llada15}, and the evidence is GSM8K only.

\begin{table}[ht]
  \centering
  \small
  \caption{Second checkpoint: LLaDA-1.5 on GSM8K, $\numpaths{=}10$, semi-autoregressive blocks of 32, 2560 evaluations, $n{=}600$ over two seeds (1200 paired observations). Bare means; higher is better for plurality, coverage (pass@10) and per-sample; distinct answers per pool and disagreement are descriptive. Bold marks the better method on each accuracy column. Paired \methodname minus self-consistency: coverage $+6.75$, plurality $+2.46$, per-sample $-1.82$. The coverage gain differs from the LLaDA-8B-Instruct gain on the same problems by $+0.58$.}
  \label{tab:llada15}
  \begin{tabular}{lrrrrr}
    \toprule
    Method & Plur. ($\uparrow$) & Cover. ($\uparrow$) & Per-s. ($\uparrow$) & Distinct & Disagr. \\
    \midrule
    Self-consistency & 76.85 & 87.25 & \textbf{74.15} & 1.69 & 17.82 \\
    \methodname, count & \textbf{79.31} & \textbf{94.00} & 72.33 & 2.26 & 28.22 \\
    \bottomrule
  \end{tabular}
\end{table}

\section{Further results}
\label{app:further}

\subsection{Position by position decoding}
\label{app:example}

Figure~\ref{fig:latticepair} shows one problem decoded twice at temperature zero, once with the penalty off and once with it on, on the canvas the samples share. A row is one position and a column is one of the ten samples. Temperature zero makes the comparison exact, since the decode consumes no randomness. With the penalty off, the ten samples are one string repeated ten times, which is Proposition~\ref{prop:degeneracy} and is the upper lattice. Every difference in the lower lattice is therefore the penalty and nothing else. A cell is blue where the sample wrote a token no peer had taken, 57.5\% of the written cells. A cell is grey where the schedule releases the penalty in the last quarter of a block, the steps left outside the penalty scope $\rho$ (Section~\ref{sec:method-sampler}), 24.1\%. The rest are a path arriving first at a position, 8.9\%, or a path writing a token a peer held, 9.5\%. A path writes a token a peer held when its margin beats the gate of eight raw logits per peer. The top rows are mostly grey, which the commit order explains. The decoder commits in order of confidence. The opening word of an answer is among the least certain positions of its block, so most samples write it last. That is exactly when the penalty has been released. All twenty samples answer correctly. The pair therefore shows ten distinct decodes through the same slots against one decode repeated ten times. Sample 9 is shifted against its peers: the canvas aligns indices and not content, so part of the pool's disagreement is a phrasing shift of that kind.

\begin{figure}[ht]
  \centering
  \includegraphics[width=\linewidth]{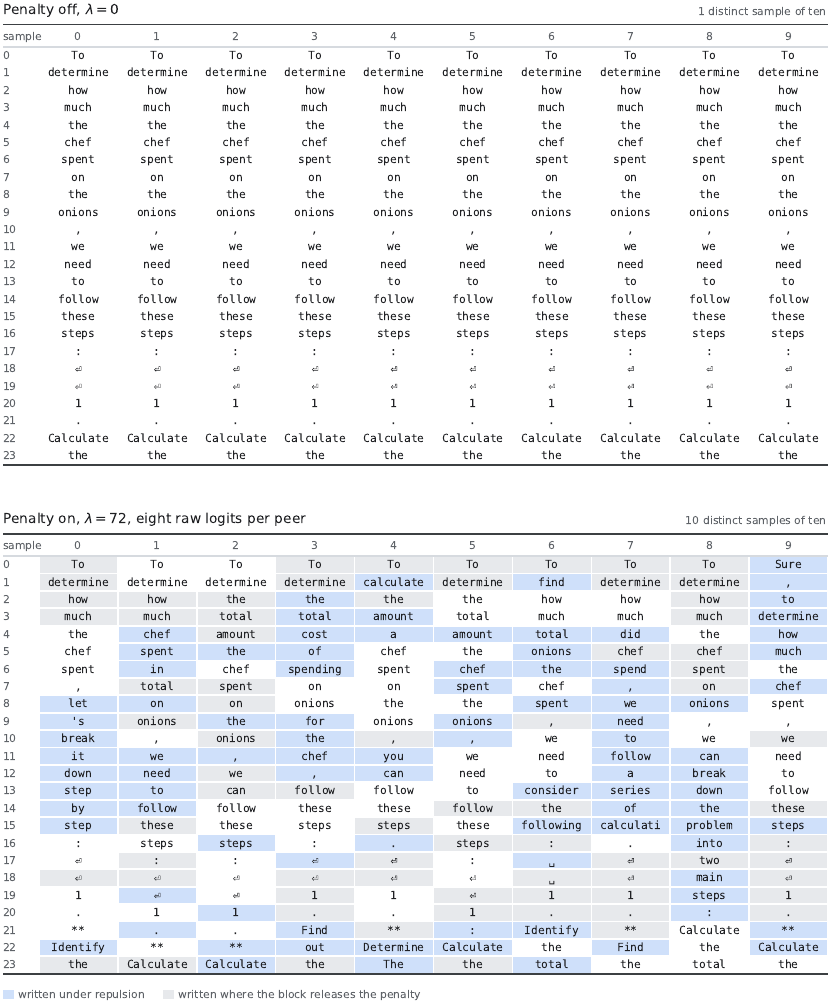}
  \caption{One problem decoded twice at temperature zero, the penalty off above and on below. A row is one position of the canvas, a column is one of the ten samples; in the lower half blue marks a token written under repulsion, grey one written where the block releases the penalty. The section text gives the rest.}
  \label{fig:latticepair}
\end{figure}

\subsection{Each family at its own best temperature.}
\label{sec:results-tuned}
On GSM8K, self-consistency improves when its temperature is tuned. We therefore re-sweep every family and compare it at its own best setting as well as at the shared 0.6.

\begin{table}[h]
  \centering
  \footnotesize
  \caption{Every family at its own temperature, 1280 network function evaluations, $\numpaths = 10$, blocks of 32. Left: plurality accuracy ($\uparrow$) with each family at a temperature its authors report, self-consistency at both 0.6 and 1.0 since it has no published setting for this model. Right: the mean paired difference of \methodname with the count statistic at $\temp = 0.6$ over that row, on the problems the two methods share. GSM8K $n{=}600$, MATH $n{=}492$, TruthfulQA $n{=}684$; two seeds throughout, five for self-consistency at 0.6 and for the count statistic on GSM8K; the $\temp = 0$ row is one deterministic decode. Baselines are our reproductions.}
  \label{tab:tuned}
  \setlength{\tabcolsep}{3.0pt}
  \begin{tabular}{l@{\hskip 9pt}ccc@{\hskip 9pt}ccc}
    \toprule
    & \multicolumn{3}{c}{Plurality at own $\temp$ ($\uparrow$)} & \multicolumn{3}{c}{Difference from \methodname} \\
    \cmidrule(lr){2-4} \cmidrule(lr){5-7}
    Method & \scriptsize GSM & \scriptsize MATH & \scriptsize TQA & \scriptsize GSM & \scriptsize MATH & \scriptsize TQA \\
    \midrule
    Self-consistency, $\temp = 0.6$ & 77.05 & 40.14 & 59.02 & $+2.06$ & $-0.10$ & $+0.76$ \\
    Self-consistency, $\temp = 1.0$ & 78.28 & 39.47 & 58.65 & $+0.83$ & $+0.57$ & $+1.12$ \\
    TAPS, $\temp = 1.0$ & 77.87 & 39.31 & 57.74 & $+1.24$ & $+0.73$ & $+2.03$ \\
    ODD ($\alpha{=}64$), $\temp = 1.0$ & 78.23 & 36.66 & 60.86 & $+0.88$ & $+3.38$ & $-1.08$ \\
    ODD ($\alpha{=}256$), $\temp = 1.0$ & 78.61 & 37.52 & 60.12 & $+0.50$ & $+2.52$ & $-0.35$ \\
    \midrule
    \methodname, count, $\temp = 0.6$ & 79.11 & 40.04 & 59.77 & --- & --- & --- \\
    \methodname, count, $\temp = 0$ & 80.38 & 40.86 & 59.84 & $-1.27$ & $-0.81$ & $-0.07$ \\
    \methodname, exp.\ count, $\temp = 0.6$ & 80.65 & 39.58 & 59.82 & $-1.54$ & $+0.46$ & $-0.05$ \\
    \methodname, coll.-wt., $\temp = 0.6$ & 80.07 & 38.81 & 60.60 & $-0.95$ & $+1.23$ & $-0.83$ \\
    \bottomrule
  \end{tabular}
\end{table}

\looseness=-1 Tuning moves the baselines by about a point and in both directions: ODD ($\alpha{=}256$) gains 0.61 on GSM8K between 0.6 and 1.0 while losing 0.23 on MATH, so the shared temperature of Table~\ref{tab:nfe} does not systematically favour us. The differences are largest on MATH against ODD and smallest on TruthfulQA, where ODD at $\alpha = 64$ leads. At 2560 evaluations the GSM8K picture is the same: $+0.68$ and $+0.63$ over the two tuned ODD settings; $+14.40$ and $+18.17$ over tuned TAPS under pure diffusion.

\subsection{Repulsion alone, at temperature zero}
\label{sec:results-argmax}

\looseness=-1 At temperature zero the penalty is the only source of a pool. Argmax paths with the penalty off are identical copies, so self-consistency collapses. We sweep the per-peer statistic $\dose/(\numpaths-1)$, the margin in raw logits a peer's commitment must overcome. We repeat the design as a sequential chain of five drafts, each repelled from those completed. The decode uses no randomness, so each cell is one deterministic decode.

\begin{table}[h]
  \centering
  \footnotesize
  \caption{The penalty at temperature zero, GSM8K, $n = 600$, blocks of 32, 256 steps. One decode per cell. Agreement is the mean fraction of sample pairs returning the same answer; distinct is the mean number of different answers in a pool. The last rows of each block are the sampled references at $\temp = 0.6$, seed 0, and the single greedy decode.}
  \label{tab:argmax}
  \setlength{\tabcolsep}{3.2pt}
  \begin{tabular}{lrrrrrrr}
    \toprule
    Configuration & Gate $\gate$ & $\dose$ & Agreement & Distinct & Per-sample ($\uparrow$) & Coverage ($\uparrow$) & Plurality ($\uparrow$) \\
    \midrule
    \multicolumn{8}{l}{\emph{Parallel, $\numpaths = 10$, the deployed design}} \\
    Penalty off & 0 & 0 & 100.00 & 1.00 & 73.00 & 73.00 & 73.00 \\
    Count penalty & 1 & 9 & 77.54 & 1.91 & 72.17 & 89.50 & 76.97 \\
    Count penalty & 2 & 18 & 75.62 & 2.04 & 71.85 & 90.83 & 77.31 \\
    Count penalty & 4 & 36 & 73.04 & 2.18 & 71.90 & 92.00 & 77.97 \\
    Count penalty & 8 & 72 & 68.27 & 2.43 & 69.65 & 92.33 & 78.60 \\
    \cmidrule(l{2pt}r{2pt}){1-8}
    \methodname, $\temp = 0.6$ & 7.1 & 64 & 69.36 & 2.40 & 71.13 & 94.33 & 78.82 \\
    Self-consistency, $\temp = 0.6$ & --- & --- & 79.85 & 1.83 & 72.53 & 88.17 & 76.33 \\
    \midrule[0.12em]
    \multicolumn{8}{l}{\emph{Sequential, five drafts, each repelled from those completed}} \\
    Penalty off & 0 & 0 & 100.00 & 1.00 & 73.00 & 73.00 & 73.00 \\
    Count penalty & 1 & 4 & 76.32 & 1.61 & 73.00 & 88.83 & 76.50 \\
    Count penalty & 2 & 8 & 76.25 & 1.63 & 73.00 & 88.50 & 76.67 \\
    Count penalty & 8 & 32 & 70.93 & 1.77 & 71.90 & 90.33 & 78.03 \\
    Count penalty & 16 & 64 & 59.93 & 2.10 & 66.33 & 90.17 & 78.12 \\
    \cmidrule(l{2pt}r{2pt}){1-8}
    Single greedy decode & --- & --- & --- & --- & 73.00 & 73.00 & 73.00 \\
    \bottomrule
  \end{tabular}
\end{table}

\looseness=-1 With the penalty off, both designs are exactly degenerate: they return 100.00\% agreement, 1.00 distinct answers; per-sample, coverage and plurality all equal to the single greedy decode's 73.00. A pool of ten therefore adds nothing over a pool of one. Every number above it comes from the peer term. Plurality is monotone in the gate in both designs. At gate 8, nearest the deployed strength, the parallel design reaches 78.60 against self-consistency's 76.33. It comes within 0.33 of the deployed sampled configuration's five-seed plurality of 78.93 (Appendix~\ref{app:abl-seeds}). Coverage still favours the sampled configuration, 93.33 against 92.33, so the penalty closes the plurality gap and not the coverage gap. The first path is unpenalised only in the sequential design, where draft one matches the greedy anchor exactly at gates 1 and 2. In the parallel design path zero commits inside the step it is read for, so it does not.

\looseness=-1 A sharper control fixes disagreement and removes the peer read while keeping the allocation law. Each sample draws its tokens in proportion to the water-filling counts of its own logits (Proposition~\ref{prop:allocation}), peer-blind, at the gate matching the penalty's disagreement (32.56 against 34.35; GSM8K in blocks, 1280 evaluations, three seeds). This control gains $0.69$ in plurality over self-consistency at $\temp = 0.6$, a gain we do not establish, while \methodname at temperature zero gains $2.46$ over the control, which we do. The allocation law alone does not carry the gain; the peer read does.

\subsection{Path-index symmetry}
\label{sec:results-symmetry}

\looseness=-1 Remark~\ref{prop:canonicity} settles the processing order in structure. Permuting an order held fixed across steps returns the same pool relabeled, so no ablation over such orders exists; an order that varies across steps is a change of information structure, which the remark leaves free. The remark leaves open whether the position a path occupies predicts whether it is correct, a statement about marginals. With the penalty on, per-sample accuracy across the ten parallel indices spreads 3.17 points at gate 8 and shows no ordering. The sequential chain, where draft $i$ is repelled from $i-1$ predecessors and a gradient exists by construction, runs 73.00, 69.67, 66.50, 63.83, 58.67 down the draft order at gate 16.

\looseness=-1 The information structure is the one order choice the remark leaves free. Table~\ref{tab:filtration} quantifies it: a triangular filtration in which path $i$ reads only paths $0 \ldots i-1$, against the deployed mutual cascade at matched gate.

\begin{table}[h]
  \centering
  \footnotesize
  \caption{The cost of the information structure. GSM8K, $n = 600$, temperature zero, one decode per cell. Path 0 of the filtration is the released greedy decode, its answers identical across the three configurations on all 600 problems and reading the anchor's 73.00; the symmetric control is the deployed mutual cascade at the same gate, decoded in the same campaign, reproducing the banked symmetric cell on all 600.}
  \label{tab:filtration}
  \setlength{\tabcolsep}{4.2pt}
  \begin{tabular}{lrrrrrr}
    \toprule
    Coupling & Gate $\gate$ & $\dose$ & Agreement & Per-sample ($\uparrow$) & Coverage ($\uparrow$) & Plurality ($\uparrow$) \\
    \midrule
    Triangular filtration & 4 & 36 & 76.53 & 71.75 & 90.00 & 76.12 \\
    Triangular filtration & 8 & 72 & 73.12 & 71.30 & 92.00 & 77.41 \\
    Triangular filtration & 16 & 144 & 67.72 & 68.67 & 91.50 & 77.19 \\
    Mutual cascade, symmetric & 8 & 72 & 68.27 & 69.65 & 92.33 & 78.60 \\
    \bottomrule
  \end{tabular}
\end{table}

\looseness=-1 At matched gate the filtration votes 1.19 below the symmetric cascade and is ahead on per-sample accuracy, 71.30 against 69.65. Its agreement at gate 16, 67.72, lands on the symmetric cascade's gate-8 value of 68.27, the halving that $\numpaths(\numpaths-1)/2$ ordered pairs against $\numpaths(\numpaths-1)$ predicts. At that matched agreement the symmetric cascade still converts better, 78.60 against 77.19. We keep the symmetric cascade deployed.

\subsection{Decomposition of the gain into coverage and selectivity}
\label{sec:results-decomposition}

\paragraph{Decomposition of the coverage gain into the per-sample term and dependence.}
\looseness=-1 \methodname's coverage gain is the per-sample term, not dependence between samples: the Fr\'echet cap of Proposition~\ref{prop:frechet} at the penalty's measured accuracies on GSM8K with blocks averages $1.94$ points, against a coverage gain of $+5.94$ over three seeds. Measured directly by a within-pool against cross-pool contrast that is exactly zero for a peer-blind sampler, the pool dependence is $+0.23$ for the penalty against $+0.02$ for self-consistency; the contrast draws three samples index-blind without replacement from one pool against one from each of three independent pools, which carry the same marginals.

\looseness=-1 One result bounds this measurement in advance. At fixed per-path marginals, re-coupling alone raises coverage above the i.i.d.\ value by at most $(1 - 1/\numpaths)^{\numpaths} < e^{-1}$, and by order $10^{-3}$ per problem at a per-path rate of one half. At the per-path rates measured on GSM8K a gain of practical size must therefore come from the induced marginals; at rates near $0.2$, as on MATH, the per-problem ceiling reaches ten points and the split must be read at the measured marginals. The net effect on the vote is therefore empirical. Plurality factors exactly as $P(\plur) = P(\passat) \cdot P(\plur \mid \passat)$, coverage times the vote's accuracy given coverage. Writing $(a,x)$ for \methodname and $(b,y)$ for self-consistency, the gain $\Delta = ax - by$ splits exactly into a coverage term $(a-b)(x+y)/2$ and a selectivity term $(x-y)(a+b)/2$, each factor's change read at the mean of the other; this symmetric attribution averages the two orderings, which differ by the interaction $(a-b)(x-y)$. This attributes every change to a wider pool or a more selective vote.

\begin{table}[h]
  \centering
  \footnotesize
  \caption{Decomposition of the gain. Each row is \methodname against self-consistency in the same regime at 2560 network function evaluations, $\numpaths = 10$, temperature 0.6, two seeds, paired on the problems both methods share. Coverage and selectivity are the two terms of the split defined in the text; they sum to $\Delta$ in every row. Fixes and breaks are the fractions of problem-seed units whose vote changes in each direction. On TruthfulQA ten uniform guesses over four options already cover the correct option with probability 94.37, which the pure-diffusion count row's coverage matches, so its coverage column is uninformative. Note: the pure-diffusion rows predate the end-of-sequence correction (Appendix~\ref{app:eosfix}).}
  \label{tab:decomp}
  \setlength{\tabcolsep}{5.0pt}
  \begin{tabular}{llrrrrr}
    \toprule
    Benchmark & Penalty & $\Delta$ plurality & Coverage & Selectivity & Fixes & Breaks \\
    \midrule
    \multicolumn{7}{l}{\emph{Semi-autoregressive blocks of 32}} \\
    GSM8K & count & $+3.92$ & $+5.27$ & $-1.35$ & 7.17 & 3.25 \\
    GSM8K & collision-wt. & $+2.92$ & $+5.98$ & $-3.06$ & 6.83 & 3.92 \\
    MATH & count & $+2.34$ & $+3.09$ & $-0.75$ & 6.00 & 3.66 \\
    MATH & collision-wt. & $+2.13$ & $+1.94$ & $+0.19$ & 5.18 & 3.05 \\
    TruthfulQA & count & $+0.88$ & $+2.36$ & $-1.48$ & 6.65 & 5.77 \\
    TruthfulQA & collision-wt. & $+3.07$ & $+3.43$ & $-0.36$ & 8.33 & 5.26 \\
    \midrule[0.12em]
    \multicolumn{7}{l}{\emph{Pure diffusion}} \\
    GSM8K & count & $+19.17$ & $+20.19$ & $-1.03$ & 21.50 & 2.33 \\
    GSM8K & collision-wt. & $+22.58$ & $+23.91$ & $-1.33$ & 24.92 & 2.33 \\
    MATH & count & $+5.69$ & $+10.25$ & $-4.56$ & 9.65 & 3.96 \\
    MATH & collision-wt. & $+11.89$ & $+13.71$ & $-1.82$ & 15.04 & 3.15 \\
    TruthfulQA & count & $+0.88$ & $+16.36$ & $-15.48$ & 5.99 & 5.12 \\
    TruthfulQA & collision-wt. & $+2.27$ & $+10.22$ & $-7.95$ & 8.99 & 6.73 \\
    \bottomrule
  \end{tabular}
\end{table}

\looseness=-1 Coverage accounts for between 51 and 95 percent of the moved mass in every row. The selectivity term is negative in eleven of the twelve. The vote is therefore never the source of the gain and usually offsets part of it. Churn is asymmetric and the asymmetry is the whole net effect. On GSM8K with blocks the penalty fixes 7.17\% of problem-seed units and breaks 3.25\%. Under pure diffusion it fixes 21.50\% and breaks 2.33\%. A rescue analysis settles a mechanism question. If the penalty helped by breaking a concentrated wrong consensus, then among the problems self-consistency gets wrong the rescue rate would rise with the concentration of its error. It falls instead, with a negative correlation in ten of the twelve rows; $+0.00$ and $+0.07$ in the other two. The problems rescued are the ones on which the baseline was scattered, not the ones on which it was confidently wrong.

\looseness=-1 Better aggregation does not help: across two waves of selectors, learned and heuristic, none beats plain plurality on the deployed configuration. The coverage-to-plurality gap is real, 79.38 against 96.33 at $\numpaths = 32$ on GSM8K (Appendix~\ref{app:abl-k}); nothing we built closes it.

\subsection{Coverage prediction}
\label{app:covpred}

\looseness=-1 A wider pool is priced in per-sample accuracy: with $\numpaths$ near-independent draws, coverage on a problem whose per-sample accuracy is $a$ is $1 - (1-a)^{\numpaths}$, so an extra unit of per-sample accuracy is worth $\numpaths (1-a)^{\numpaths - 1}$ of coverage, the probability that no other sample would have found the answer. At $\numpaths = 10$ that rate is 10 at $a = 0$ and $10^{-8}$ at $a = 0.9$. It is a falsifiable curve and the decodes test it.

\looseness=-1 A cross-benchmark reading appears to refute it. MATH has far lower per-sample accuracy than GSM8K, 31.60 against 72.31, and the smaller coverage term, $+3.09$ against $+5.27$ in blocks, which is the opposite of ``largest where the model rarely succeeds''. That reading confounds two things. The rate multiplies the per-sample shift the penalty produces, and that shift is not constant across benchmarks. The test has to hold the benchmark fixed and vary $a$, which is what Figure~\ref{fig:covpred} does: problems are stratified by the baseline's per-sample accuracy, and each stratum's measured coverage change is set against the change independence predicts at that stratum's measured per-sample shift, $(1-a_{\mathrm{SC}})^{\numpaths} - (1-a_{\methodname})^{\numpaths}$.

\begin{figure}[h]
  \centering
  \includegraphics{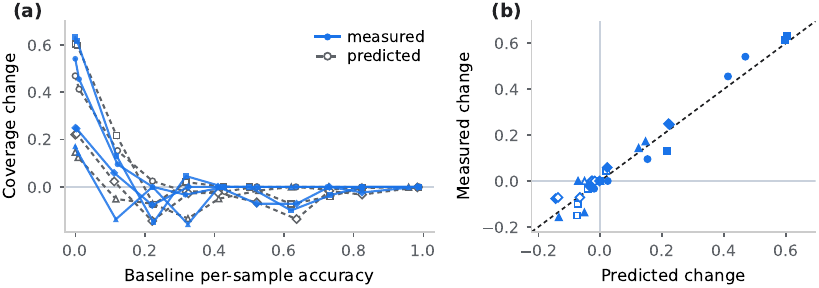}
  \caption{The exchange rate against the decodes. Problems are stratified by the baseline's per-sample accuracy; each point is one stratum of one benchmark and regime, at 2560 network function evaluations, $\numpaths = 10$, temperature 0.6, two seeds, paired per problem. Predicted is $(1-a_{\mathrm{SC}})^{\numpaths} - (1-a_{\methodname})^{\numpaths}$, the coverage the stratum's measured per-sample shift buys if the samples were independent. \textbf{(a)} coverage change against baseline accuracy, measured and predicted. \textbf{(b)} the same points, measured against predicted, with the identity line; open markers are strata of fewer than 15 problems. Circles and squares are GSM8K in blocks and under pure diffusion, triangles and diamonds are MATH.}
  \label{fig:covpred}
\end{figure}

\looseness=-1 The account survives, and it survives quantitatively rather than in direction only. Summed over strata and weighted by their size, the ratio of measured to predicted coverage change is 1.07 and 1.02 on GSM8K in the two regimes and 1.11 and 1.16 on MATH: between two and sixteen percent more coverage than the per-sample shift buys under independence, which is the residual dependence measured directly in Section~\ref{sec:results-decomposition} and is small beside the shift itself. The concentration the rate predicts is pronounced. On GSM8K in blocks the stratum the baseline never solves holds 61 problems and gains $0.541$ of coverage; every stratum above $a = 0.35$ gains $0.000$ to three decimals. Under pure diffusion the same stratum holds 204 problems and gains $0.632$.

\looseness=-1 This also settles the MATH question without appeal to the vote. At $a = 0$ the rate is $\numpaths$ on both benchmarks, so the rate is not what differs; the per-sample shift is. The penalty moves that stratum by $+0.113$ on GSM8K and by $+0.027$ on MATH, a factor of four, and the coverage gains follow at $0.541$ and $0.172$. MATH gains less coverage because the penalty finds fewer new answers there, not because the pool prices them differently.

\looseness=-1 The same fact is visible without any model, in how many of the $\numpaths$ traces are correct. The penalty empties both ends of that distribution and fills the middle. On GSM8K in blocks the share of problems no trace solves falls from 0.130 to 0.067 and the share every trace solves falls from 0.555 to 0.413; under pure diffusion the two fall from 0.383 to 0.137 and from 0.435 to 0.157. The first movement is the coverage gain, since coverage is one minus the share with no correct trace, and the second is what the vote pays for it, since a problem that leaves unanimity can lose its plurality. On MATH in blocks the share at zero barely moves, 0.461 to 0.435, which is the same shortfall the strata show and the reason its coverage gain is small.

\subsection{Reading the two sweeps}
\label{app:grids}

\looseness=-1 Three readings of Table~\ref{tab:grids} are not available, and the table marks none of them. Panel B's plurality column has no readable maximum: from $\dose = 32$ to $384$ it spans 1.4 points, which is why bold marks temperature's best rungs only. Its per-sample marginal exceeding its own $\dose = 0$ value is a mixture effect of self-consistency's bimodality in this regime and supports no claim on its own; the per-sample edge is the one read at matched disagreement, which Section~\ref{sec:results-mechanism} states. And the penalty's disagreement saturates rather than ending where the sweep does, moving 52.01 to 53.46 as $\dose$ doubles from 192 to 384, which is what a gate already past hard exclusion predicts; temperature has no such ceiling and reaches 88.87 at $\temp = 3.0$, by destroying the pool.

\subsection{Pass@$k$}
\label{app:passk}

\looseness=-1 Coverage in the tables above is pass@10; Table~\ref{tab:passk} and Figure~\ref{fig:passk} report the whole curve for three samplers at 2560 evaluations, and Table~\ref{tab:passkmain} the four-point curve for every method of Table~\ref{tab:nfe} at its own campaigns; both use the estimator $1 - \binom{n-c}{k}/\binom{n}{k}$ over the $n = 10$ samples per problem \citep{chen2021evaluating}, which is unbiased for independent samples. The shape is the same in all four cells: \methodname trails self-consistency at $k = 1$ in three of them, which is per-sample accuracy and the price already discussed, crosses by $k = 2$, and the gap widens monotonically in $k$ from there. Under pure diffusion the collision-weighted statistic is ahead at every $k$, including $k = 1$. MATH, where the vote converts least, is $+4.57$ at $k = 10$ in the deployed regime: the pool improves even where plurality cannot convert it, and the curve is the ceiling any selector over the pool could reach. TruthfulQA is excluded, since four options give a chance curve of $1 - 0.75^{k}$, already 94.37 at $k = 10$.

\begin{table}[h]
  \centering
  \footnotesize
  \caption{Pass@$k$ at 2560 network function evaluations, $\numpaths = 10$, seeds 0 and 1 pooled; GSM8K $n{=}600$, MATH $n{=}492$. Left: that estimator at four $k$; bold is the best per column within a regime. Right: the count statistic's pass@$k$ gap over self-consistency, labeled per cell. The pure-diffusion configurations predate the end-of-sequence correction (Appendix~\ref{app:eosfix}), which moves the baseline about two points and neither the crossover nor the ordering for $k \geq 2$.}
  \label{tab:passk}
  \begin{minipage}[t]{0.63\linewidth}
    \vspace{0pt}
    \centering
    \setlength{\tabcolsep}{1.6pt}
    \scriptsize
    \renewcommand{\arraystretch}{0.92}
    \begin{tabular}{l!{\color{black!20}\vrule width 0.5pt}rrrr!{\color{black!20}\vrule width 0.5pt}rrrr}
      \toprule
      & \multicolumn{4}{c}{GSM8K} & \multicolumn{4}{c}{MATH} \\
      Sampler & \scriptsize $k{=}1$ & \scriptsize 2 & \scriptsize 5 & \scriptsize 10 & \scriptsize $k{=}1$ & \scriptsize 2 & \scriptsize 5 & \scriptsize 10 \\
      \midrule
      \multicolumn{9}{l}{\emph{Blocks of 32}} \\
      Self-consistency & \textbf{72.18} & 78.58 & 84.54 & 87.67 & \textbf{31.60} & 39.60 & 48.78 & 55.89 \\
      \methodname, count & 70.93 & \textbf{81.52} & 90.07 & 93.83 & 30.22 & \textbf{40.51} & \textbf{52.12} & \textbf{60.47} \\
      \methodname, coll.-wt. & 66.70 & 80.14 & \textbf{90.20} & \textbf{94.75} & 28.25 & 38.57 & 50.75 & 58.74 \\
      \midrule
      \multicolumn{9}{l}{\emph{Pure diffusion}} \\
      Self-consistency & 53.42 & 56.96 & 60.53 & 63.17 & 23.77 & 27.99 & 32.87 & 36.08 \\
      \methodname, count & 56.19 & 70.82 & 82.40 & 86.83 & 23.37 & 32.68 & 44.08 & 51.52 \\
      \methodname, coll.-wt. & \textbf{63.06} & \textbf{76.23} & \textbf{86.41} & \textbf{91.25} & \textbf{27.79} & \textbf{37.30} & \textbf{48.46} & \textbf{55.79} \\
      \bottomrule
    \end{tabular}
  \end{minipage}\hfill
  \begin{minipage}[t]{0.35\linewidth}
    \centering
    \vspace{0pt}
    \includegraphics[width=\linewidth]{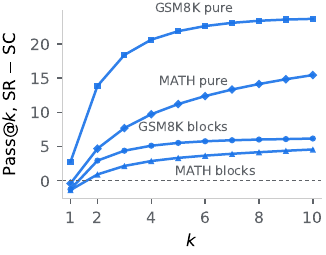}
    \refstepcounter{figure}\label{fig:passk}
    {\footnotesize\textbf{Figure~\thefigure:} the pass@$k$ curves of Table~\ref{tab:passk}; \methodname crosses self-consistency by $k{=}2$ and the gap widens thereafter.}
  \end{minipage}
\end{table}

\begin{table}[h]
  \centering
  \scriptsize
  \caption{Pass@$k$ for every method of Table~\ref{tab:nfe}, same campaigns and settings; pass@1 and pass@10 are Table~\ref{tab:nfe}'s per-sample and coverage columns. TruthfulQA reports $k{=}1$ only, chance alone reaching 94.37 at $k{=}10$; TAPS ($\dagger$) runs the uncorrected protocol. Bold: best per column within a regime.}
  \label{tab:passkmain}
  \setlength{\tabcolsep}{2.5pt}
  \renewcommand{\arraystretch}{0.82}
  \begin{tabular*}{\linewidth}{@{\extracolsep{\fill}}l!{\color{black!20}\vrule width 0.5pt}cccc!{\color{black!20}\vrule width 0.5pt}cccc!{\color{black!20}\vrule width 0.5pt}c@{}}
    \toprule
    & \multicolumn{4}{c}{GSM8K} & \multicolumn{4}{c}{MATH} & TQA \\
    Method & \scriptsize $k{=}1$ & \scriptsize 3 & \scriptsize 5 & \scriptsize 10 & \scriptsize $k{=}1$ & \scriptsize 3 & \scriptsize 5 & \scriptsize 10 & \scriptsize $k{=}1$ \\
    \midrule
    \multicolumn{10}{l}{\emph{Semi-autoregressive blocks of 32}} \\
    Self-consistency & \textbf{70.04} & 82.45 & 86.03 & 89.70 & \textbf{30.02} & 44.14 & 49.78 & 56.81 & \textbf{54.44} \\
    TAPS & 66.77 & 83.28 & 87.53 & 91.00 & 29.26 & 44.03 & 49.84 & 57.32 & 53.53 \\
    ODD ($\alpha{=}64$) & 64.94 & 82.60 & 87.21 & 91.75 & 24.19 & 40.12 & 46.45 & 54.27 & 52.35 \\
    ODD ($\alpha{=}256$) & 66.57 & 83.32 & 87.73 & 92.17 & 22.77 & 39.22 & 45.97 & 54.37 & 49.85 \\
    \cmidrule(l{2pt}r{2pt}){1-10}
    \methodname, count & 68.21 & 85.06 & 89.50 & 93.53 & 27.67 & \textbf{44.22} & \textbf{50.98} & \textbf{59.86} & 48.07 \\
    \methodname, count, $\tau{=}0$ & 69.33 & 85.67 & 90.02 & 94.33 & 27.13 & 43.56 & 49.80 & 57.11 & 46.92 \\
    \methodname, exp.\ count & 66.47 & \textbf{85.74} & \textbf{90.61} & \textbf{94.67} & 26.78 & 43.24 & 50.28 & 59.65 & 50.75 \\
    \methodname, collision-wt.\ & 61.87 & 84.72 & 90.15 & 94.42 & 24.44 & 40.75 & 47.46 & 55.39 & 52.78 \\
    \midrule[0.12em]
    \multicolumn{10}{l}{\emph{Pure diffusion (whole response as one block)}} \\
    Self-consistency & 54.47 & 61.55 & 63.95 & 67.00 & 22.87 & 29.06 & 32.37 & 37.40 & 54.17 \\
    TAPS$^\dagger$ & 50.34 & 61.75 & 65.58 & 69.75 & 22.46 & 30.81 & 34.57 & 39.74 & \textbf{55.26} \\
    \cmidrule(l{2pt}r{2pt}){1-10}
    \methodname, count & 53.85 & 75.65 & 81.87 & 87.75 & 22.62 & 36.53 & 42.66 & 50.51 & 53.79 \\
    \methodname, count, $\tau{=}0$ & 53.82 & 75.52 & 81.68 & 87.17 & 23.17 & 36.71 & 42.64 & 50.41 & 52.85 \\
    \methodname, exp.\ count & 53.87 & 77.53 & 84.31 & \textbf{90.33} & 20.47 & 34.43 & 40.73 & 48.78 & 50.78 \\
    \methodname, collision-wt.\ & \textbf{59.81} & \textbf{80.14} & \textbf{85.26} & 90.00 & \textbf{25.10} & \textbf{40.38} & \textbf{47.11} & \textbf{55.79} & 50.50 \\
    \bottomrule
  \end{tabular*}
\end{table}

\subsection{End-of-sequence correction}
\label{app:eosfix}

\looseness=-1 The pure-diffusion rows of Table~\ref{tab:nfe} apply a remedy the model's authors attach to this regime: LLaDA sets the confidence of the end-of-sequence token to zero during pure-diffusion sampling on GSM8K, MATH and three other benchmarks, because the padding of its supervised data otherwise ends a decode early \citep{nie2025llada}. Two questions attach to the correction. What is it worth, since our earlier campaigns omit it; and does the penalty's margin merely substitute for it, since with the content mask off peers' committed end-of-sequence tokens enter the count and the penalty is pushed off that token. Table~\ref{tab:eosfix} decides both: every configuration applies the correction, and one configuration additionally turns the content mask on, which removes the token from the count entirely so the penalty cannot substitute for anything.

\begin{table}[h]
  \centering
  \footnotesize
  \caption{The regime re-decoded with the released correction applied to every configuration. GSM8K under pure diffusion, $\numpaths = 10$, 128 steps, temperature 0.6, $n = 600$, two seeds. The third row additionally excludes the end-of-sequence token from the count. The rows with the mask off are the same campaign as Table~\ref{tab:nfe}'s pure-diffusion block; the first row is the baseline. The last column is the difference in plurality from that baseline.}
  \label{tab:eosfix}
  \setlength{\tabcolsep}{5pt}
  \begin{tabular}{llrrrrr}
    \toprule
    Sampler & Mask & Plurality ($\uparrow$) & Coverage ($\uparrow$) & Per-sample ($\uparrow$) & Disagr. & $\Delta$ plur. \\
    \midrule
    Self-consistency & off & 57.02 & 67.00 & 54.47 & 20.40 & --- \\
    \addlinespace
    \methodname, count & off & 71.53 & 87.75 & 53.85 & 49.73 & $+14.50$ \\
    \methodname, count & \textbf{on} & 70.44 & 88.00 & 53.07 & 50.45 & $+13.42$ \\
    \methodname, coll.-wt.\ & off & 76.27 & 90.00 & 59.81 & 45.93 & $+19.25$ \\
    \bottomrule
  \end{tabular}
\end{table}

\looseness=-1 The correction is material, and it does not explain the margin. It lifts the baseline from the 54.39 of Table~\ref{tab:abl-blocks} to 57.02 and removes the truncation the omission caused: no sample here fails to parse, against 0.64\% and 0.81\% of samples in the uncorrected pair. The count penalty's margin over the baseline moves by only 1.41 with the correction, from 15.91 to 14.50, so about a point of the uncorrected pure-diffusion margin was ours rather than the sampler's; the 2.63 is the baseline's own lift. What remains is not two points. With the correction applied to both sides the collision-weighted penalty leads by 19.25, and with the end-of-sequence token removed from the count altogether, so that the penalty has no access to the channel in question, the count statistic still leads by 13.42 on the vote and by 21.00 on coverage. Turning the mask on costs 1.08 points of plurality, so the channel carries none of the effect.

\looseness=-1 The same campaign covers all three benchmarks and populates Table~\ref{tab:nfe}'s pure-diffusion block. Against the corrected baseline the margins are $+13.96$ to $+19.25$ on GSM8K and $+5.40$ to $+11.37$ on MATH; on TruthfulQA the temperature-zero configuration leads by $3.42$. Two protocols remain uncorrected: TAPS, whose released loop exposes no hook for the correction, and the 2560-evaluation campaigns of Tables~\ref{tab:grids} and~\ref{tab:main}.

\end{document}